\RequirePackage{fix-cm}
\documentclass[smallcondensed]{svjour3}     
\smartqed  
\usepackage{graphicx}
\usepackage{mathtools}
\usepackage[linesnumbered,ruled,vlined]{algorithm2e}
\usepackage{algorithmic}
\usepackage{hyperref}
\usepackage{amsmath}
\usepackage{xcolor}
\usepackage{hhline}
\usepackage{multirow}
\newcommand{\R}{\mathbb{R}}
\usepackage{amsfonts}
\usepackage{comment}

\newcommand{\deltaback}{\operatornamewithlimits{\Delta}}

\newcommand{\deltaforward}{\operatornamewithlimits{\Delta}}

\begin{document}

\title{A Forward Backward Greedy approach for Sparse Multiscale Learning
}


\author{Prashant Shekhar         \and
        Abani Patra 
}


\institute{Prashant Shekhar \at
              Data Intensive Studies Center, Tufts University \\
              \email{prashant.shekhar@tufts.edu}           
           \and
           Abani Patra \at
              Department of Mathematics and Computer Science, Tufts University\\ \email{abani.patra@tufts.edu} 
}

\date{Received: date / Accepted: date}

\maketitle

\begin{abstract}
Multiscale Models are known to be successful in uncovering and analyzing the structures in data at different resolutions. In the current work we propose a feature driven Reproducing Kernel Hilbert space (RKHS), for which the associated kernel has a weighted multiscale structure. For generating approximations in this space, we provide a practical forward-backward algorithm that is shown to greedily construct a set of basis functions having a multiscale structure, while also creating sparse representations from the given data set, making representations and predictions very efficient. We provide a detailed analysis of the algorithm including recommendations for selecting  algorithmic hyper-parameters and estimating probabilistic rates of convergence at individual scales. Then we extend this analysis to multiscale setting, studying the effects of finite scale truncation and quality of solution in the inherent RKHS. In the last section, we analyze the performance of the approach on a variety of simulation and real data sets, thereby justifying the efficiency claims in terms of model quality and data reduction.
\end{abstract}

\section{Introduction}
\label{intro}
Hierarchical and multiscale models have been used in variety of scientific domains as a means to model either a physical phenomena (based on modeling the governing equations through discretization procedures like finite element \cite{efendiev2009multiscale,efendiev2013generalized}) or modeling discrete data from a physical system by assuming multilevel interactions among the observations resulting in hierarchical features \cite{ferrari2004multiscale,maggioni2016multiscale,allard2012multi}. In this regard, the recent advances in Deep learning \cite{goodfellow2016deep,lecun2015deep} architectures have provided us novel directions for hierarchical modeling of datasets through sequential function compositions. Each layer in these networks represents a particular level in the overall hierarchy of the approximation produced at the output layer. Although widely successful, the capability of these models to produce approximations at unseen data samples (referred to as generalization) has been a major area of concern since these highly flexible approximators are easily prone to overfitting (resulting mainly from excessive number of stacked hidden layers). In this paper, we instead propose
multiscale approximations to a given dataset that exploit and uncover the inherent hierarchy of interactions among different data points. Since the depth of our approximation (number of scales) is dependent on data itself (we can stop the algorithm when it reaches a given budget for the size of sparse representation or mean squared error loss for noise free datasets), hence it need not be separately analyzed as is the case for number of hidden layers in a deep neural network. Moreover, the scales we include in our approximation have a geometric reasoning based on the varying support structure of the associated basis functions. Such type of intuition is generally missing while adding additional layers in a deep neural network.

The problem setup we consider can be briefly described as follows. Given a dataset $D = \{(x_i,y_i)\}_{i=1}^n$ composed of observations $Y = (y_1,y_2,....,y_n) \in {\R}^n$ at locations $X = (x_1,x_2,...,x_n) \in {\R}^{n \times d}$, we consider (and approximate) some true underlying latent process $f: {\R}^d \to {\R}$ generating these observations. We also assume the observations $(x_i,y_i) \in X \times Y$ as i.i.d samples from $\Omega_x \times \Omega_y$ ($X \subset \Omega_x \subset {\R}^d$ and $Y \subset \Omega_y \subset {\R}$). Given such a dataset, generally in machine learning, we aim to find an approximation $\hat{f} \in \mathcal{H}$ (native Reproducing Kernel Hilbert Space - RKHS) to $f$. This is obtained by solving
\begin{equation}
    \hat{f} = arg \min_{\tilde{f} \in \mathcal{H}} L(\tilde{f}) = arg \min_{\tilde{f} \in \mathcal{H}} \mathbb{E}[(y -\tilde{f}(x))^2]
\end{equation}

Here the expectation operator( $\mathbb{E}[\cdot]$) is applied over the sampling distribution that resulted in the dataset $D = \{(x_i,y_i)\}_{i=1}^n$. However, since we don't have access to this full distribution, we usually work with the empirical estimate of loss quantified by mean squared error (MSE)

\begin{equation}
\label{eq_cost}
    \hat{f} = \arg \min_{\tilde{f} \in \mathcal{H}} \frac{1}{n}\sum_{i=1}^n (y_i - \tilde{f}(x_i))^2 
\end{equation}

In machine learning and statistics literature \cite{williams2006gaussian,scholkopf2002learning,schaback2006kernel,fasshauer2007meshfree}, researchers usually work with  standard kernel functions (for the native RKHS) like Squared Exponential, Matern, or compactly supported functions such as Wendland functions \cite{wendland2004scattered}. However, such an approach implements a uniformity constraint on the structure of the basis functions throughout the approximation domain, leading to inefficient local approximations. For this reason,  kernels are often constructed to capture the properties of sophisticated approximation spaces under consideration (for example constructing kernels from feature maps \cite{scholkopf2002learning}). In this research we work with a multiscale kernel of the form (\ref{multiscale kernel}), that is constructed as a weighted combination of function evaluations at multiple scales.

\begin{equation}
\label{multiscale kernel}
    K(x,y) = \sum_{s \in \mathcal{I}} \zeta_s \sum_{j=1}^{n} \psi^s_j(x) \psi^s_j(y)
\end{equation}
Here $s$ is a scale parameter that takes values from a countable set $\mathcal{I} = \{0,1,2,..\}$. $\psi^s_j(x)$ is a function in RKHS $\mathcal{H}_s$ (centered at $x_j \in X$ and evaluated at $x \in \Omega_x$) with squared exponential function $(\exp(-r^2/\kappa_s))$ as its associated kernel. Here $r$ is the distance between the data points and $\kappa_s$ is the length scale of the kernel which adapts with the scale parameter s (as will be shown in the following sections). Also, $\zeta_s$ is a positive, scale dependent weight, such that the series $\sum_{s \in \mathcal{I}} \zeta_s$ converges. The first summation in (\ref{multiscale kernel}) explores the different scales, while the second summation evaluates functions at a particular scale. For practical algorithms, since we cannot compute (\ref{multiscale kernel}) upto infinite scales, we truncate this kernel intelligently based on the structure in the data and provide a greedy iterative algorithm that generates the required approximations. \cite{opfer2006multiscale,opfer2006tight} introduced the idea of multiscale kernels and showed their application with compactly supported and refinable set of functions. Later \cite{griebel2015multiscale} studied the approximation properties of such kernels. Instead of focusing on general approximation spaces governed by a family of multiscale kernels, we introduce and go deeper into the analysis of a very specific multiscale kernel (\ref{multiscale kernel}), that is constructed from functions taken from several standard RKHS with a varying support structure. This enables us to model multiscale interactions in a dataset. 

Using a multiscale model to capture the structure in data has been discussed in the literature. \cite{allard2012multi,maggioni2016multiscale} explore the idea of Geometric Multiresolution Analysis (GMRA) where they look at multiscale analysis from a geometric point of view and construct a subdividision tree of the observations based on exponentially decaying weights. This involves learning a low dimensional geometric manifold from data in a multiscale fashion. However, their cell size was fixed throughout the domain. This can lead to sub-optimal approximations. \cite{liao2016adaptive} provided an adaptive version of GMRA that included adaptive partitioning for better performance. However, this still is very different from the idea which we explore here, which involves constructing adaptive basis functions at multiple scales that capture the structure in the data belonging to different inherent scales. Hence instead of dividing the approximation domain, we analyze it at different resolutions globally.

A related problem that has been analyzed in the literature pertains to segmenting data into components generated by different models (analogous to scale for our case). \cite{ma2007segmentation} assumes a mixture of Gaussians as an underlying model and optimizes data segmentation. \cite{vidal2005generalized,chen2013robust,lu2006combined,ma2008estimation} focus on subspace learning based on low-rank approximations.

Going back to the multiscale kernel defined in (\ref{multiscale kernel}), we note that, in the inner summation we have considered $n$ functions at each scale. However, in practice, the basis formed by translates of a kernel functions can be very ill-conditioned \cite{de2010stability,fasshauer2009preconditioning} (depending on the data distribution). So, our proposed approach implements a greedy forward-backward strategy for intelligently choosing good bases at each scale. Our forward greedy selection is a modification of the commonly used matching pursuit algorithm \cite{mallat1993matching,jones1992simple,elad2010sparse}, also referred to as boosting in the machine learning literature \cite{buehlmann2006boosting}. \cite{tropp2004greed,donoho2005stable} analyzed the performance of matching pursuit algorithms for good approximations. Since forward greedy algorithms by themselves can lead to good approximations but inefficient basis selection \cite{zhang2011adaptive}, we also implement a backward deletion of functions  at the end of forward selection at each scale  \cite{couvreur2000optimality}. It should be noted that while selecting the basis functions intelligently at each scale, we also sample small set of data points that form the center of these functions. The union of these small sets, coupled with scale information over all the scales till convergence is regarded as our sparse representation $C^{\omega}$, where $\omega$ is the convergence scale for the dataset D. This sparse representation (as will be shown in the following sections ) enables us to make any predictions without going back to the full dataset D.

Since we sample functions intelligently from each scale till convergence, if we theoretically consider all functions at every scale simultaneously as a dictionary of functions, then choosing the final set of functions is analogous to the problem of dictionary learning \cite{mairal2009supervised,gribonval2015sample} in signal processing. Sparse methods for dictionary based algorithms have additionally proved to be effective for processing image and sound data \cite{mysore2012block,elad2010sparse}. However, since our dictionary has a very strong group structure (based on scales), direct application of these methods to our problem is inefficient.

Adaptive column sampling algorithms are another set of approaches that have close relations with the proposed method. This includes random column selection \cite{williams2000using,gittens2013revisiting}, non-deterministic \cite{gittens2013revisiting} and deterministic \cite{farahat2011novel} adaptive selection. Using clustering methods such as K-means to find new rank-K space projections \cite{zhang2008improved} is another line of thinking  that has been explored before. Although these methods do provide efficient ways to sample independent columns from a matrix, our approach is geared more towards approximating observations $y$, rather than just feature selection. 

Regarding optimal sparse representations, K-sparse models have attracted a lot of attention in the machine learning and signal processing literature. For a given dictionary, here the target is to express each datapoint through the span of atmost K elements of the dictionary. \cite{maurer2010k} provided a Hilbert space analysis of the problem. \cite{candes2007dantzig,donoho2006compressed,lewicki2000learning} are some examples of work that analyzed such sparse approximations through overcomplete dictionaries.

 Summarizing, this paper makes the following contributions:
\begin{enumerate}
    \item We propose a multiscale approximation space that acts as a native space for approximations produced from functions belonging to a set of different RKHS with kernels of varying support structure. This proves helpful in analyzing the properties of such approximations under a unified umbrella.
    \item We propose a Forward-Backward greedy approach to construct these multiscale approximations, that is shown to intelligently pick scales that are relevant for capturing the structure in any given dataset, while totally ignoring irrelevant scales. The algorithm also enables data reduction and fast predictions by utilizing the selected multiscale sparse representations.
    \item Intelligently deciding the values of large number of user dependent hyper-parameters inherent in greedy approaches is a well known problem \cite{zhang2011adaptive}. We tackle this by explicitly deriving recommendations for the single inherent hyperparameter ($\epsilon_0$) for our algorithm (Proposition \ref{epsilonrate}). This is coupled with a detailed single-scale and multiscale analysis of the approximations produced. The analytical results are further justified by the experiments in the results section. 
\end{enumerate}

\section{Multiscale Approximation space}

Let $\mathcal{I}$ be a countable set of scales to be considered in the analysis. For the current research, we consider the set of non-negative integers $\mathbb{Z}^{\geq 0}$ ($\mathcal{I} = \{0,1,2.. \}$)

\begin{lemma}
\label{summability}
For each scale $s \in \mathcal{I}$, consider $\zeta_s > 0$ such that the series $\sum_{s \in \mathcal{I}} \zeta_s$ converges. Then for every function $\psi^s_j(\cdot)$ in a RKHS $\mathcal{H}_s$ and evaluated at $x \in \Omega_x$ 
($\psi^s_j(x) = K_s(x,x_j)$, where $K_s$ is the associated kernel, $x_j \in X$ and $|X| = n$) , the summability condition
\begin{equation}
    \sum_{s \in \mathcal{I}} \sum_{j=1}^{n} \zeta_s |\psi^s_{j}(x)|^2 < \infty
\end{equation}
follows.
\end{lemma}

\begin{proof}
We need to consider two situations

Firstly, assuming x is one of the centers (i.e. $x \in X$). In this case for all s we have one term in the inner summation with $\psi_j^s(x) = 1$ (assuming the normalization $K_s(0,0) = 1$, where $K_s(\cdot,\cdot)$ is the reproducing kernel for $\mathcal{H}_s$). Hence $\forall x{'} \in X \setminus x$, we have $\psi_j^s(x{'}) < 1$. Thus
\[
\sum_{j=1}^n |\psi^s_j(x)|^2 < n \implies \sum_{s \in \mathcal{I}} \zeta_s \sum_{j=1}^n |\psi^s_j(x)|^2 < \sum_{s \in \mathcal{I}} \zeta_s n < \infty
\]
The last inequality comes from the convergence condition on series $\sum_{s \in \mathcal{I}} \zeta_s$.

Secondly, if x is not one of the centers, then all individual terms in $\sum_{j=1}^n |\psi^s_j(x)|^2$ will be less than 1. Hence the bound will naturally hold.\qed
\end{proof}

\subsection{Multiscale Kernel}

With the scale dependent functions $\psi^s$ across the scales belonging to the set $\mathcal{I}$, the function
\begin{equation}
\label{kernel2}
    K(x,y) = \sum_{s \in \mathcal{I}} \zeta_s \sum_{j \in r_s} \psi^s_j(x) \psi^s_j(y)
\end{equation}
is the multiscale kernel we target in our work. It should be noted here that unlike (\ref{multiscale kernel}), here the second summation goes only over the elements of $r_s$ instead of full $X$ ($|r_s| \leq n$), where $r_s$ represents the set of indices for linearly independent functions at scale s. Additionally, at each scale, our proposed greedy algorithm just selects a subset of the these linearly independent functions to be a part of the final approximation (as for a given dataset, if the observations lie in a subspace of the space spanned by the full set of independent functions, then we don't need the full set to model our data). However, instead of defining the kernel this way (\ref{kernel2}), if we take the second summation to n for each scale, then like \cite{opfer2006multiscale} we will have to introduce a projection operator on the weights of basis functions for enforcing uniqueness of approximation (see Lemma 2.4 in \cite{opfer2006multiscale}). Hence using such a representation simplifies our notation.

\begin{figure}[h]
\centering
\includegraphics[width=11cm]{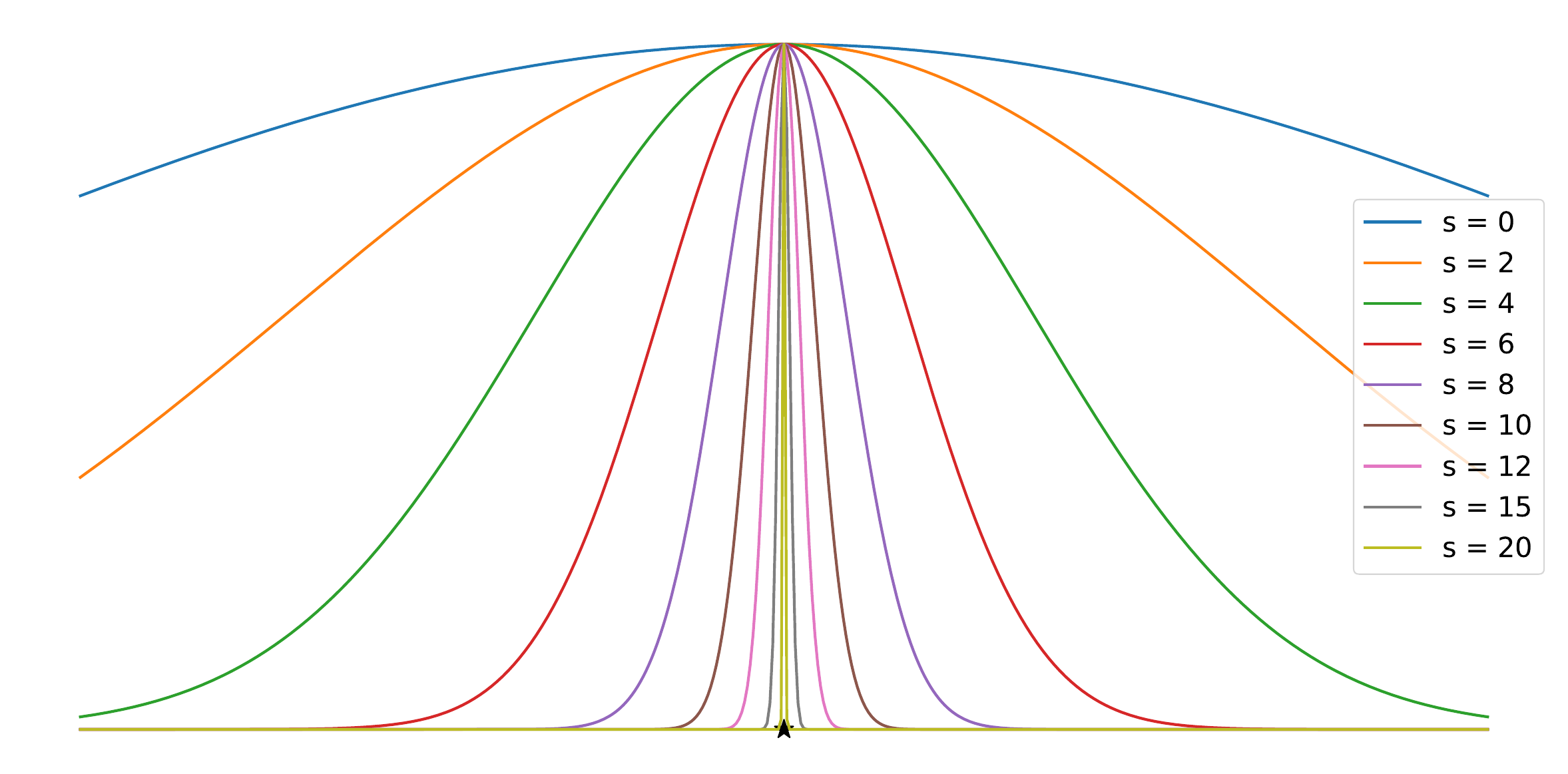}
\caption{Behavior of the functions centered at a data point (shown by star marker at the bottom) with increase in scale.}
\label{basis}
\end{figure}

Here, each of our scale dependent RKHS $\mathcal{H}_s$ (to which the function $\psi^s_j$ belong to) is uniquely determined by the positive definite kernel function indexed by scale (s) as 

\begin{equation}
\label{K^sT}
    K_s(a,b) = exp\Big(-\frac{||a -b||^2}{\kappa_s} \Big),\quad \kappa_s = T/2^s
\end{equation}
where T is a constant which is of the order of square of the radius of the dataset D (\ref{T}). The dependence on scale is expressed as an exponent of 2 (motivated from the multiresolution analysis through wavelets \cite{ferrari2004multiscale,opfer2006multiscale,bermanis2013multiscale}). This also enforces a structured increase in the rate of decay of the functions from their center (see Fig. \ref{basis} for illustration of this behavior with increasing s). 

\begin{theorem}
The multiscale kernel K in (\ref{kernel2}) is a positive definite function on arbitrary subset $\Omega_x$ of $\R^d$.
\end{theorem}

\begin{proof}
Let $\alpha \in \R^n$ be any n-dimensional vector of values
\begin{align*}
    \sum_{i=1}^n \sum_{j=1}^n \alpha_i \alpha_j K(x_i,x_j) &= \sum_{i=1}^n \sum_{j=1}^n \alpha_i \alpha_j \sum_{s \in \mathcal{I}} \zeta_s \sum_{m \in r_s} \psi^s_m(x_i) \psi^s_m(x_j) \\
    &= \sum_{s \in \mathcal{I}} \zeta_s \sum_{i=1}^n \sum_{j=1}^n \alpha_i \alpha_j \sum_{m \in r_s} \psi^s_m(x_i) \psi^s_m(x_j) \\
    &= \sum_{s \in \mathcal{I}} \zeta_s \sum_{i=1}^n \sum_{j=1}^n \alpha_i \alpha_j \Big< \psi^s_{i}(\cdot), \psi^s_{j}(\cdot)\Big>_{\R^{|r_s|}}\\
    &= \sum_{s \in \mathcal{I}} \zeta_s \Big< \sum_{i=1}^n \alpha_i\psi^s_{i}(\cdot), \sum_{j=1}^n \alpha_j \psi^s_{j}(\cdot)\Big>_{\R^{|r_s|}}\\
    &= \sum_{s \in \mathcal{I}} \zeta_s \Big|\Big|\sum_{i=1}^n \alpha_i \psi^s_{i}\Big|\Big|_2^2
\end{align*}
Here the third equality uses the fact $\psi^s_m(x_i) = \psi^s_i(x_m)$. Now, since we have $\sum_{s \in \mathcal{I}} \zeta_s \Big|\Big| \sum_{i=1}^n \alpha_i \psi^s_{i} \Big|\Big|_2^2 \geq 0$ (using $\zeta_s > 0$), hence kernel K is a positive definite function (using the definition in Theorem 3.2 in \cite{fasshauer2007meshfree}). 
\qed
\end{proof}

\begin{remark}
From \cite{shekhar2020hierarchical,bermanis2013multiscale}, we know that the numerical rank of the Gaussian kernel increases monotonically with s (due to the narrower support at higher scales), hence after sufficiently high s($\geq \omega$), $||\sum_{i=1}^n \alpha_i \psi_i^s||_2^2 >0$ ($=0$ only with $\alpha = 0$), making the Gramian matrix $K(x_i,x_j)_{i,j = 1}^n$ numerically positive definite. This property is later used to determine the truncation scale for the multiscale kernel.
\end{remark}

\subsection{Native approximation space for K}
In this subsection we introduce the native approximation space for the multiscale kernel in which we will produce our approximations. Following the ideas in \cite{opfer2006multiscale}, we introduce the following Hilbert space. 

\begin{definition}
We define the Hilbert space $l^2_{\zeta}$ using the weighted $l^2$ sequence
\begin{equation}
 \Big\{\{ \theta_s^j \in \R \}_{j \in r_s, s \in \mathcal{I}}, \   such\ that\ \sum_{s \in \mathcal{I}} \sum_{j \in r_s} \frac{|\theta^j_s|^2}{\zeta_s} < \infty \Big\}.
\end{equation}
and equipped with the inner product
\[
\Big<\alpha, \beta \Big>_{l^2_{\zeta}} = \sum_{s \in \mathcal{I}} \sum_{j \in r_s} \frac{\alpha_s^j \beta^j_s}{\zeta_s}
\]\qed
\end{definition}

Using $l^2_{\zeta}$ we can then generate a function space

\[
M_{\psi} (\alpha) = \sum_{s \in \mathcal{I}} \sum_{j \in r_s} \alpha^j_s \psi^s_{j}, \quad where\ \alpha = \{\alpha_s^j\} \in l^2_{\zeta}
\]

Now from \cite{opfer2006multiscale} we have the following result

\begin{theorem}
\label{mpsi}
The space $range(M_{\psi})$ equipped with the inner product
\begin{equation}
    \Big< \sum_{s \in \mathcal{I}} \sum_{j \in r_s} \alpha^j_s \psi^s_{j}, \sum_{s \in \mathcal{I}} \sum_{j \in r_s} \beta^j_s \psi^s_{j} \Big>_K = <\alpha, \beta >_{l^2_{\zeta}}
\end{equation}
is a Hilbert space.
\end{theorem}

Consider again the space $\mathcal{H}_s$ for scale specific approximations. Since, the functions sampled from the kernel $K_s(\cdot,x_j) (x_j \in X)$ can't be directly used as a basis set (as it is ill-conditioned), so here we introduce $\mathcal{H}_s$ in terms of linearly independent functions sampled from the set of trial functions $D^s_X = span\{K_s(\cdot,x_i): x_i
\in X\}$

\begin{definition}
For all  $s \in \mathcal{I}$, we define subspaces
$
\mathcal{H}_s = \Big\{ \sum_{j \in r_s} \alpha^j_s \psi^s_j, \alpha_s^j \in \R \Big\}$ with set \{$\psi^s_{\cdot}$\} forming a bases in $\mathcal{H}_s$. We further equip them with the inner products
\[
    \Big< \sum_{j \in r_s} \alpha^j_s \psi^s_{j}, \sum_{j \in r_s} \beta^j_s \psi^s_{j} \Big>_{\mathcal{H}_s} = \sum_{j \in r_s} \alpha_s^j \beta_s^j
\]
\end{definition}

Using the space $\mathcal{H}_s$ in the construction of the native space for the multiscale kernel we obtain the following result.

\begin{theorem}
\label{native}
Consider $f_s \in \mathcal{H}_s$ such that $\sum_{s \in \mathcal{I}} \frac{1}{\zeta_s} ||f_s||^2_{\mathcal{H}_s} < \infty$. Then the space $\mathcal{H} = \Big\{ f: \Omega_x \to \R:\ f = \sum_{s \in \mathcal{I}} f_s with\ f_s \in \mathcal{H}_s \Big\}$
equipped with the inner product
\begin{equation}
    \Big<f,g\Big>_{\mathcal{H}} = \sum_{s \in \mathcal{I}} \frac{1}{\zeta} <f_s, g_s>_{\mathcal{H}_s}
\end{equation}
(assuming the representation $g = \sum_{s \in \mathcal{I}} g_s$) is a RKHS and the native space of the multiscale kernel K.

\end{theorem}

\begin{proof}
Starting with the space $\mathcal{H}$
\[
f = \sum_{s \in \mathcal{I}} f_s = \sum_{s \in \mathcal{I}} \sum_{j \in r_s} \alpha^j_s \psi^s_j = M_{\psi} (\alpha), \quad \text{where }\alpha \in l^2_{\zeta}
\]
Also, for the given condition we have,
\[ \sum_{s \in \mathcal{I}} \frac{1}{\zeta_s} ||f_s||^2_{\mathcal{H}_s} = \sum_{s \in \mathcal{I}} \frac{1}{\zeta_s} \sum_{j \in r_s} (\alpha^j_s)^2 < \infty
\]
Now, considering the given inner product for the space $\mathcal{H}$
\[
<f,g>_{\mathcal{H}} = \sum_{s \in I} \frac{1}{\zeta_s} <f_s,g_s>_{\mathcal{H}_s} = \sum_{s \in \mathcal{I}} \frac{1}{\zeta_s} \sum_{j \in r_s} \alpha_s^j \beta^j_s
\]
\[
\implies <f,g>_{\mathcal{H}} = <M_{\psi}(\alpha), M_{\psi}(\beta) >_K = <\alpha, \beta >_{l^2_{\zeta}}
\]
Hence $\mathcal{H}$ is the alternate representation of the space $(range(M_{\psi}), <\cdot, \cdot>_K)$ and hence a Hilbert space (by Theorem \ref{mpsi}). Now for $\mathcal{H}$ to be a RKHS, we just have to prove that the evaluation functional operator on $\mathcal{H}$ is bounded (i.e. there exists some $M > 0$ such that, $|L_x(f)| \leq M ||f||_{\mathcal{H}}$ ). Therefore $\forall f \in \mathcal{H}$ (assuming $L^s_x$:  evaluation operator on $\mathcal{H}_s$)
\begin{align*}
    |L_x(f)| &= |f(x)| = \Bigg|\sum_{s \in \mathcal{I}} \sum_{j \in r_s} \alpha^j_s \psi^s_{j}(x) \Bigg| \leq \sum_{s \in \mathcal{I}} \Bigg|\sum_{j \in r_s} \alpha^j_s \psi^s_j (x) \Bigg| = \sum_{s \in \mathcal{I}} |L^s_x(f_s)|\\
    & \leq \sum_{s \in \mathcal{I}} M^s ||f_s||_{\mathcal{H}_s} \quad (\text{$\mathcal{H}_s$ is a RKHS and $M^s$ a bounding constant})\\
    &= \sum_{s \in \mathcal{I}}M^s \sqrt{\sum_{j \in r_s} (\alpha^j_s)^2} = \sum_{s \in \mathcal{I}} M^s \zeta_s \sqrt{\sum_{j \in r_s} \frac{(\alpha^j_s)^2}{(\zeta_s)^2}} \quad (\text{since $\zeta_s >0 \forall s$})\\
    & \leq \sum_{s \in \mathcal{I}} M^s \zeta_s \sum_{j \in r_s} \frac{|\alpha^j_s|}{|\zeta_s|} \quad (\text{since $||\cdot||_2 < ||\cdot||_1$ in $\R^n$})\\
    & \leq \frac{max(M^s \zeta_s)}{min(|\alpha^j_s|)} \sum_{s \in \mathcal{I}} \sum_{j \in r_s} \frac{|\alpha_s^j|^2}{\zeta_s} 
\end{align*}
Identifying $\sum_{s \in \mathcal{I}} \sum_{j \in r_s} \frac{|\alpha_s^j|^2}{\zeta_s}$ as $||f||^2_{\mathcal{H}}$ concludes our proof for showing boundedness of  $L_x$. Hence $\mathcal{H}$ is a RKHS. Now to show K is its reproducing kernel, consider the function $K(x,\cdot) = \sum_{s \in \mathcal{I}}\zeta_s \sum_{j \in r_s} \psi^s_{j}(x) \psi^s_j(\cdot)$, we need to show $K(x,\cdot) \in \mathcal{H}$ which follows from
\[
||K(x,\cdot)||^2_\mathcal{H} = \sum_{s \in \mathcal{I}}\sum_{j \in r_s} \frac{\zeta_s \psi^s_j(x) \zeta_s \psi^s_j(x)}{\zeta_s} = \sum_{s \in \mathcal{I}} \sum_{j \in r_s} \zeta_s |\psi^s_j(x)|^2 < \infty
\]

Here the finiteness of summation comes from Lemma \ref{summability}. Now to show the reproducing property, let $f = \sum_{s \in \mathcal{I}} \sum_{j \in r_s} \alpha^j_s \psi^s_j$, then we have
\begin{align*}
    <f,K(x,\cdot)>_{\mathcal{H}_s} &= \Big< \sum_{s \in \mathcal{I}} \sum_{j \in r_s} \alpha^j_s \psi^s_{j},\sum_{s \in \mathcal{I}} \sum_{j \in r_s} \zeta_s \psi^s_{j}(x)\psi^s_{j}(\cdot) \Big>\\
    &= \sum_{s \in \mathcal{I}} \sum_{j \in r_s} \frac{\alpha^j_s \zeta_s \psi^s_{j}(x)}{\zeta_s}\\
    &= f(x)
\end{align*}
Thus establishing $K$ as the reproducing kernel of $\mathcal{H}$, thereby concluding the proof.
\qed
\end{proof}

\section{Multiscale Algorithm}
The multiscale kernel introduced in the previous section consider scales from a countable set $\mathcal{I}$. However, for a practical algorithm we need to truncate the set to a computationally manageable upper limit on the scale. Fig. \ref{basis} shows the behavior of the these functions with increasing scales. From the extremely narrow support (numerically) of these functions at higher scales (approaching a Dirac delta behavior) it can be shown that at such scales additional basis functions will not improve the approximation, justifying the finite scale truncation. In fact, the result presented in Lemma \ref{truncation1} shows that if truncation is carried out at a sufficiently large scale such that the Gramian matrix obtained from the truncated kernel is numerically positive definite, then the weights obtained for such a multiscale kernel will be identical to the weights for the full kernel (with infinite scales). For clarity, we will denote the truncated kernel as $K^{\omega}$, where $\omega$ is the truncation scale.

In the literature \cite{opfer2006multiscale,griebel2015multiscale}, typically for approximations in the native space of the multiscale kernels, the corresponding Gramian matrix (after suitable truncation) is assembled to formulate the linear system, and then solved for coefficients. However, since in this work our aim is to also exploit the redundancy in this basis set (ultimately leading to data reduction), we avoid solving the interpolation problem with the Gramian matrix $K^{\omega}(x_i,x_j)_{i,j =1}^n$, and directly work with greedily constructing the approximations that have the following form.
\begin{align*}
    Af^{\omega}(\cdot) &= \sum_{i=1}^n\phi^{\omega}_i K^{\omega}(x_i,\cdot) = \sum_{i=1}^n \phi^{\omega}_i \Bigg( \sum_{s = 0}^{\omega} \zeta_s \sum_{j \in r_s}\psi_j^s(x_i) \psi_j^s(\cdot) \Bigg)\\
    &= \sum_{s=0}^{\omega} \sum_{j \in r_s} \Bigg(\sum_{i=1}^n \zeta_s \phi^{\omega}_i \psi_j^s(x_i) \Bigg) \psi_j^s(\cdot) = \sum_{s=0}^{\omega} \sum_{j \in r_s} \alpha_s^j \psi_j^s(\cdot)
\end{align*}

In this section we introduce our multiscale algorithm which operates on a dataset $D = [X,y]$, providing an approximation to the underlying function $f$ expressed as $\sum_{s = 0}^{\omega} \sum_{x_j \in X_s} \alpha_s^j \psi^s_j$. Here the inner summation is over $X_s$, representing the set of points, for which the associated functions (centered at these points) were chosen to a part of the final approximation, i.e. we have $|X_s| \leq |r_s|$, where $|r_s|$ denotes the number of numerically independent functions at scale s. Hence, instead of solving for $\phi^{\omega}_i$ (weights in the approximation $\sum_{i = 1}^n \phi^{\omega}_i K^{\omega}(x_i,\cdot)$), we directly solve for coefficients ($\alpha^j_s$) of the functions $\psi_j^s$ at each scale s. This also has the benefit of not having to decide the weights $\zeta_s$ at each scale, which was the case for \cite{opfer2006multiscale}.  Since at each scale translates of kernel functions provide a highly redundant set of functions for approximation, so we implement a modification of forward greedy strategy for selecting a good set of functions for approximation. However, following the ideas of \cite{zhang2011adaptive} for more effective sparsification of our basis at each scale, the forward selection algorithm is followed by the backward deletion procedure. The motivation for such a procedure comes from the fact that during forward selection due to the greedy nature of the algorithm, we might sample more than the required number of functions for good approximation. The backward procedure cleans out these additional functions, that may provide independent information but are not numerically useful for modeling the given data.
\IncMargin{1em}
\begin{algorithm}
\SetKwData{Left}{left}\SetKwData{This}{this}\SetKwData{Up}{up}
\SetKwFunction{Union}{Union}\SetKwFunction{FindCompress}{FindCompress}
\SetKwInOut{Input}{Input}\SetKwInOut{Output}{Output}
\Input{Data: $D = (X,y) \subset \Omega_x \times \Omega_y$\\
Truncation scale: $\omega$\\
Starting tolerance: $\epsilon_0$\\
}
\Output{Sparse representation: $C^{\omega}$ \\
Model weights: $\Theta^{\omega}$
\emph{\noindent\rule{10.1cm}{0.6pt}}
}
\emph{Initialize: $s=0, B^s = [\cdot],\Theta^s = [\cdot], C^s = [\cdot], f^{s-1} = 0,t_s = y$}\\
Compute supporting parameters: $\gamma = \frac{\epsilon_0 \vartheta_0^2}{||t_0||_2}, \Delta = \frac{\epsilon_0^2\vartheta_0^2}{n}$\ ($\vartheta_s$ as in (\ref{vartheta}))\\
\While{$s \leq \omega$}{
$B_s,\Theta_s,C_s,\mathcal{E} \leftarrow Forward\_selection(\epsilon_s,D,s,t_s)$\\
$B_s,\Theta_s,C_s \leftarrow Backward\_deletion(\epsilon_s,B_s,\Theta_s,C_s,\mathcal{E},t_s)$\\
\If{$C_s \neq \{\phi\}$}{
$B^s,\Theta^s,C^s \leftarrow Update\_model(B^s,\Theta^s,C^s,B_s,\Theta_s,C_s)$\\
$f_s \leftarrow B_s\Theta_s$\\
}
\Else{
$f_s \leftarrow \textbf{0}$\\
}
$\mathcal{E}^s \leftarrow MSE(t_s - f_s)$\\
$f^s \leftarrow f^{s-1} + f_s$\\
$s \leftarrow s + 1$\\
$t_s \leftarrow t_{s-1} - f_{s-1}$\\
$\epsilon_s \leftarrow \max \Big( \frac{\gamma ||t_s||_2}{\vartheta_s^2},\frac{\sqrt{n \Delta}}{\vartheta_s}\Big)$\\
}
\textbf{Return} $[C^{\omega},\Theta^{\omega}, \mathcal{E}^{\omega}, f^{\omega}]$
\caption{$Multiscale\  Algorithm$}\label{Algo1}
\end{algorithm}\DecMargin{1em}

The complete algorithm has been presented as Algorithm \ref{Algo1}, \ref{Algo2}, \ref{Algo3} and \ref{Algo4} in this section. For the rest of the paper, we denote the basis matrix at scale s as $B_s$ (hence $B_s = [b^j_s]_{x_j \in X_s}$, where $b^j_s = \psi^s_j|_X$ are n-dimensional column vectors with j as a proxy for function centered at $x_j \in X_s$). While moving up the scales, useful columns till scale s will be concatenated into a matrix $B^s$ (hence $B^s = [B_0,B_1,...,B_s]$). We will follow similar notations for weights of these functions ($\Theta_s, \Theta^s$). At any scale, we denote the sparse representation as $C_s$ ($|X_s| \times (d+2)$ matrix). $C_s$ is composed of the x and y coordinate of the chosen points (centers of chosen functions) along with the current scale in the last column as an additional coordinate. Like the notation for the bases, sparse representation till scale s is represented as $C^s$. For notational convenience, we will denote the number of data points in sparse representation as $|C_s|$ for scale s and $|C^s|$ for cumulative scales till scale s. Hence, we have $|C_s| = |X_s|$ and $|C^s| = \sum_{i=0}^s|X_i|$.

The multiscale algorithm (Algorithm \ref{Algo1}) accepts 2 hyperparameters. Here $\omega$ is the maximum scale (starting from 0) we are willing to explore, and will be determined based on K-fold cross validation or directly analyzing the MSE of reconstruction and truncating according to an error budget for noise free datasets. Alternatively, can also truncate based on the accumulated size of the sparse representation ($C^s$). The other hyperparameter ($\epsilon_0$) provides a lower bound on the absolute value of the normalized inner product for any chosen function and the remaining residual (Proposition \ref{epsilonrate} provides directions for properly initializing $\epsilon_0$). We provide more information on these hyperparameters later on in this section. The input dataset is in the form of X and y, with $X \in \R^{n \times d}$ being the d-dimensional coordinates of data locations, and $y \in \R^n$ are the observations made. The algorithm begins by initializing $s$ to $0$, the current available approximation $f^{s-1}$ to a null $n-dim$ vector, and the current target $t_s$ to $y$. Before moving onto the main part of the algorithm, two key supporting parameters ($\gamma, \Delta$) are initialized (more details on these parameters will be provided in Theorem \ref{set_epsilon}). Basically, these parameters ensure that whenever a new function is added to the basis set, the reduction in MSE is lower bounded (through $\Delta$), while also ensuring that the resulting basis is well conditioned (through $\gamma$). These parameters are used later on to update the threshold parameter $\epsilon_{s}$ for the subsequent scales in a controlled manner (more details in Proposition \ref{epsilonrate}). Here, in step 2 of Algorithm \ref{Algo1}, we use the quantity $\vartheta_s$, which is defined as follows and plays a very crucial role in our overall analysis
\begin{equation}
\label{vartheta} \vartheta_s = \min_j ||b^j_s||_2
\end{equation}

Hence at any scale s, $\vartheta_s$ is the norm of the function centered at the most segregated location (for it to be minimum norm depending on the decay behavior of $\psi^s_j$) in the overall approximation domain. As an example, for uniformly distributed samples, the chosen function will be located at the edge of the domain. The main part of Algorithm \ref{Algo1} consists of a loop over scales, till the counter reaches $\omega$. At each scale, we firstly select the basis that are found to be suitable for modeling $t_s$ (implemented as $Forward\_selection(\cdot)$ in Algorithm \ref{Algo3}). Then, any unnecessary functions are removed by the subroutine $Backward\_deletion(\cdot)$ shown as Algorithm \ref{Algo4}. This forward-backward procedure sometime leads to a situation where no functions are chosen, making the current scale redundant. However, if some new function do get chosen ($B_s,C_s$ are non-empty), we update our global basis set $B^s$ by appending $B_s$ to current $B^s$. Similar appends are made to the coordinate of projection $\Theta^s$ and sparse representation $C^s$ through the $Update\_model(\cdot)$ subroutine. Following this, the MSE ($\mathcal{E}^s$), and the current approximation ($f^s$) are sequentially updated. After increasing the scale counter, we then update the target $t_s$ and finally $\epsilon_s$ (using $\gamma$ and $\Delta$ computed in step 2) for the next scale as (motivation behind such an update in provided in Remark \ref{rem_set_epsilon}.1)
\begin{equation}
    \epsilon_s = \max\Bigg\{ \frac{\gamma||t_s||_2}{\vartheta_s^2} ,\frac{\sqrt{n \Delta}}{\vartheta_s}   \Bigg\}
\end{equation}

This update leads to a structured increase in the value of $\epsilon_s$ (Proposition \ref{epsilonrate}) which makes the algorithm more selective as we move up the scales.

\IncMargin{1em}
\begin{algorithm}
\SetKwData{Left}{left}\SetKwData{This}{this}\SetKwData{Up}{up}
\SetKwFunction{Union}{Union}\SetKwFunction{FindCompress}{FindCompress}
\SetKwInOut{Input}{Input}\SetKwInOut{Output}{Output}
\Input{Sparse representation: $C^{\omega} = [C_0, C_1,...,C_{\omega}]$\\
Model weights: $\Theta^{\omega} = [\Theta_0, \Theta_1,...,\Theta_{\omega}]$\\
Normalizing constant: $T \in \R$\\
Test data: $X^t \in \R^{t \times d}$
}
\Output{Prediction: $P^t \in \R^{t}$ \\
\emph{\noindent\rule{10.1cm}{0.6pt}}
}
\emph{Initialization: $s = 0,P^t = \textbf{0} \in \R^t$}\\
\While{$s \leq \omega$}{
Compute length scale: $\kappa_s \leftarrow T/2^{s},\quad \text{(T defined in  (\ref{T}))}$\\
\For{$x_j$ in $C_s$}{
Compute distance vector $d \in \R^t$, between all $ x \in X^t$ and $x_j \in C_s$\\
Compute basis function: $b_s^j \leftarrow \exp \Big(-\frac{d^2}{\kappa_s}\Big)$\\
Update the prediction: $P^t \leftarrow P^t + b^j_s \theta^j_s, \quad (\theta^j_s \in \Theta_s)$\\
}
$s \leftarrow s + 1$
}
\textbf{Return} $P^t$
\caption{$Prediction(C^{\omega}, \Theta^{\omega}, X^t$)}\label{Algo2}
\end{algorithm}\DecMargin{1em}

At termination, Algorithm 1 returns $C^{\omega}$ and $\Theta^{\omega}$ as the final output, that is sufficient to make any future predictions without going back to the full dataset D. We show the steps involved in making predictions in Algorithm \ref{Algo2}. Besides accepting the sparse representation ($C^{\omega}$) and model weights ($\Theta^{\omega}$) computed in Algorithm \ref{Algo1}, Algorithm \ref{Algo2} requires the normalizing constant T (shown in (\ref{T})) and the data locations $X^t$ at which predictions are required. The algorithm begins by initiating the scale to 0 and current prediction $P^t$ to a null vector. This is followed by a a nested \textit{for} loop within a \textit{while} loop. The outer loop here covers all the scales, while the inner loop extracts basis functions from individual scales. For every $x_j \in C_s$ (depending on the way $C_s$ is defined $C_s = [X_s, y_s, s]$, when we say $x_j \in C_s$, we mean $x_j \in X_s$), $b_s^j$ are sequentially constructed, and using the corresponding weight $\theta_s^j \in \Theta_s$, we update the approximation $P^t$.

\begin{remark}
    For Algorithm 2, we have denoted the input as $C^{\omega}$ and $\Theta^{\omega}$. However, we don't need to wait till scale $\omega$ to make predictions. Algorithm \ref{Algo2} can be executed using $C^s$, $\Theta^s$ at any scale.
\end{remark}

With the basics of multiscale learning algorithm and prediction using sparse representation explained, we now go back to selection and deletion of basis at a particular scale. Starting with Algorithm \ref{Algo3}, we first describe the forward selection of functions at any scale s. Here $\epsilon_s$ is the tolerance level (lower bound) for the magnitude of projection of the current residual on the candidate column $b^j_s$. Additionally, input $t_s$ is the target to be approximated at the current scale.

\IncMargin{1em}
\begin{algorithm}
\SetKwData{Left}{left}\SetKwData{This}{this}\SetKwData{Up}{up}
\SetKwFunction{Union}{Union}\SetKwFunction{FindCompress}{FindCompress}
\SetKwInOut{Input}{Input}\SetKwInOut{Output}{Output}
\Input{Current scale and tolerance level: $s, \epsilon_s$\\
Data: $D$\\
Current target: $t_s$\\
}
\Output{Assembled basis and weight set: $B_s, \Theta_s$\\
Sparse representation: $C_s$\\
Current MSE: $\mathcal{E}$\\
\emph{\noindent\rule{10.1cm}{0.6pt}}
}

$Initialize$: $K_s = [b_s^i]_{i=1}^N$ using (\ref{T}) and  (\ref{K^sT}), $r = t_s, B_s = [\cdot],C_s =[\cdot], \Theta_s = [\cdot]$\\
Get the best column from $K_s$: $j \leftarrow \arg \min_i \{ \min_{z_i}||z_i\cdot b_s^i - r||_2^2$\} (use eq. (\ref{colchoose}))\\
Compute absolute weight for the added column $b_s^j$: $z_j = \frac{|r^T b_s^j|}{||b_s^j||_2^2}$\\
\While{$z_j \geq \epsilon_s$}{
\If{$B_s\ empty$}{
Initialize basis and sparse representation: $B_s \leftarrow [b^j_s]$, $C_s \leftarrow [x_j,y_j,s]$\\
Compute projection coordinate: $\Theta_s \leftarrow \frac{B_s^T(t_s)}{B_s^TB_s}$\\
}
\Else{
Append the chosen functions: $B_s,C_s \leftarrow Update\_add(B_s,C_s,b^j_s,x_j,y_j,s)$\\
Update projection coordinate of $t_s$ on $B_s$: $\Theta_s$ using (\ref{theta_s}) and (\ref{theta_s2})\\
}
Update residual: $r \leftarrow t_s - B_s\Theta_s$\\
Choose the next column: $j \leftarrow \arg \min_i \{ \min_{z_i}||z_i\cdot b_s^i - r||_2^2$\\
Compute $z_j$ like in step 3\\
}
Compute error: $\mathcal{E} \leftarrow MSE(r)$\\

\textbf{Return} $[B_s,\Theta_s,C_s,\mathcal{E}]$
\caption{$Forward\_selection(\epsilon_s, D, s, t_s)$}\label{Algo3}
\end{algorithm}\DecMargin{1em}

Algorithm \ref{Algo3} begins with computation of the kernel matrix ($K_s$) from the kernel defined in (\ref{K^sT}). Here T (normalizing constant) is set to be

\begin{equation}
    \label{T}
    T = 2(Diameter(X)/2)^2
\end{equation}
with $Diameter(X)$ being the maximum distance between any pair of data points in X. This value of $T$ for the kernel in (\ref{K^sT}), enables numerically global support of the functions at initial scales. This allows capturing of global changes at smaller scales, with local/high frequency changes modeled by `narrower' functions at higher scales. With this kernel, we then choose the best column for modeling the current residual $r$. In essence, solution to step 2 in Algorithm \ref{Algo3} is to choose the column which maximizes the following term

\begin{equation}
    \label{colchoose}
    j = \arg \max_i \frac{|r^T b^i_s|^2}{||b^i_s||_2^2} 
\end{equation}

We then extract the chosen $j^{th}$ column from $K_s$ ($b^j_s$), and if the magnitude of projection on this added column (computed in step 3) is sufficiently large, we choose $b^j_s$ to be the first vector in our basis set. We also initialize the sparse representation with the data point $x_j,y_j$ at which the $j^{th}$ function is centered. An additional column is added in $C_s$ for the scale number (as s is required for uniquely determining this basis function). This is followed by the computation of the coordinate of projection $\Theta_s$ for the target $t_s$ on $B_s$ that is composed of a single function. We then update the residual $r$ and choose the next candidate vector $b^j_s$ to be added to the current basis set $B_s$. This procedure is then essentially repeated until $z_j$ goes below the tolerance level $\epsilon_s$ for the chosen vector $b^j_s$. Here, $Update\_add(\cdot)$ in line 9 appends $b^j_s$ to the current basis set $B_s$. Additionally, it appends $[x_j, y_j, s]$ to the sparse representation $C_s$.

Here, we have chosen to put a tolerance on the inner product as a termination criterion because it ensures that our basis remains well conditioned for ordinary least square computation, and also the reduction of MSE with the addition on $b^j_s$ is lower bounded (Theorem \ref{set_epsilon}). This is crucial for our algorithm as kernel functions which constitute our basis are highly redundant and ill-conditioned (specially at the initial scales \cite{shekhar2020hierarchical}). This is fundamentally different than the traditional forward greedy method \cite{elad2010sparse}, that chooses a function from the dictionary that is most correlated with the current residual (without any additional checks), and puts a tolerance on the residual as a termination criterion. Similar residual based tolerance is also not suitable for our approach because it is difficult to pre-determine the amount of information in the data from a particular scale without running the risk of eventually choosing a function that makes the bases ill-conditioned. 

For updating the coordinate of projection each time, instead of solving the full least squares problem we pose it as:
\[
\min_{\Theta_s} \Big|\Big| \left[
\begin{array}{cc}
     B_s & b^j_s
\end{array}
\right] \Theta_s - t_s \Big|\Big|_2^2
\]
which leads to
\begin{equation}
\label{theta_s}
    \Theta_s = \left[
\begin{array}{cc}

     B_s^TB_s & B_s^Tb^j_s \\
     {b^j_s}^TB_s & {b^j_s}^Tb^j_s
\end{array}
\right]^{-1} \left[
\begin{array}{c}

     B_s^Tt_s \\
     {b^j_s}^Tt_s
\end{array}
\right]
\end{equation}
which can be implemented using the following matrix inversion Lemma \cite{elad2010sparse} and taking into account $(B_s^TB_s)^{-1}$ available from the last time a function was added to $B_s$

\begin{equation}
\label{theta_s2}
    \left[
\begin{array}{cc}
     M_0 & b_0\\
     b_0^T & c
\end{array}
\right]^{-1} = \left[
\begin{array}{cc}
     M_0^{-1} + p M_0^{-1} b_0 b_0^T M_0^{-1} & -pM_0^{-1} b_0\\
     -p b_0^TM_0^{-1} & p
\end{array}
\right]
\end{equation}

where $p = 1/(c - b_0^TM_0^{-1} b_0)$. Since at the beginning of Algorithm \ref{Algo3}, we begin with just one basis function, hence overall we never perform a formal matrix inversion. Coordinate computation through this lemma is implemented on line 10 in Algorithm \ref{Algo3}. In this way we keep on assembling $B_s$, until we can no longer find any useful column to further improve the approximation. Once we hit the tolerance level $\epsilon_s$, we return the quantities $B_s, \Theta_s, C_s$ and the current MSE $\mathcal{E}$.

\IncMargin{1em}
\begin{algorithm}
\SetKwData{Left}{left}\SetKwData{This}{this}\SetKwData{Up}{up}
\SetKwFunction{Union}{Union}\SetKwFunction{FindCompress}{FindCompress}
\SetKwInOut{Input}{Input}\SetKwInOut{Output}{Output}
\Input{Current scale tolerance level: $\epsilon_s$\\
Forward selection results: $B_s,\Theta_s,C_s$\\
Current MSE: $\mathcal{E}$\\
Current scale target: $t_s$\\
}
\Output{Refined results: $B_s, \Theta_s, C_s$\\
\emph{\noindent\rule{10.1cm}{0.6pt}}
}
\While{True}{
Compute least important index: $j \leftarrow \arg \min_i \{|\theta_s^i|\cdot||b_s^i||_2\}$\\
Update the sets: $B_{s0},C_{s0},\Theta_{s0} \leftarrow Update\_del(B_s,C_s,\Theta_s,t_s,j)$\\
Compute residual: $r = t_s - B_{s0} \Theta_{s0}$\\
\If{$MSE(r) - \mathcal{E} \leq ({\vartheta_s^2\epsilon_s^2})/{n}$}{
$B_s,\Theta_s,C_s \leftarrow B_{s0}, \Theta_{s0}, C_{s0}$\\
}
\Else{
    $break$\\
}
}
\textbf{Return} $[B_s,\Theta_s,C_s]$
\caption{$Backward\_deletion(\epsilon_s,B_s,\Theta_s,C_s,\mathcal{E},t_s)$}\label{Algo4}
\end{algorithm}\DecMargin{1em}

Once we obtain the basis matrix $B_s$ for the current scale from Algorithm \ref{Algo3}, we move forward to check if we can prune this selected set, while maintaining its effectiveness to approximate y (shown as Algorithm \ref{Algo4}). The necessity and importance of this backward procedure was justified in \cite{zhang2011adaptive} by the argument that with a greedy procedure for feature selection, it is always a possibility that we might select more functions than the required number to reach the current approximation accuracy. This has to do with the order in which the functions are selected. The only disadvantage \cite{zhang2011adaptive} mentioned for backward procedure, was its inefficiency for the case with a relatively very high number of functions as compared to the number of data points. This is because for such a case, the model is already highly overfit and there is no clear preference regarding which function to remove first. In our case, we fortunately don't have that problem, because our number of columns are strictly less than or equal to the number of data points (as $B_s$ returned from Algorithm \ref{Algo3} can have atmost n columns). 

Algorithm \ref{Algo4} focuses on the idea of sequentially removing the least important columns and then updating the approximation each time until the MSE starts to increase considerably (again controlled by the $\epsilon_s$). We define the least important column as a column j, deletion of which leads to the least increase in MSE. For this we have the following result

\begin{lemma}
\label{backdel_th0}
For backward deletion, the best column to remove satisfying

\begin{equation}
    \label{back}
    j = \arg \min_i MSE(t_s - B_s \Theta_{i,s}^{*}),\quad \text{where  $\Theta_{i,s}^{*} = \Theta_s - \theta_s^i e_i$}
\end{equation}

(with $e_i = [0,0,..,1,..0] \in \R^{|\Theta_s|}$ having 1 at $i^{th}$ location and $\theta_s^i \in \Theta_s$) is given by the index that minimizes the product of the absolute value of current coefficients and corresponding norm of the basis set.
\end{lemma}

\begin{proof}
    We start by defining
    \[
    Q(i) = MSE(t_s - B_s \Theta_{i,s}^{*}) = \frac{1}{n}||t_s - (B_s \Theta_s - b_s^i \theta_s^i)||_2^2 
    \]
    
    where $b_s^i$ is a column in $B_s$ (centered at $x_i$) with $\theta_s^i$ being the corresponding coordinate in $\Theta_s$. Hence
    \begin{align*}
        Q(i) &= \frac{1}{n}||t_s - B_s \Theta_s||_2^2 + \frac{1}{n}||b_s^i \theta_s^i||_2^2 + \frac{2}{n}(t_s - B_s \Theta_s)^T(b_s^i\theta_s^i)
    \end{align*}
    denoting $t_s - B_s \Theta_s$ as the residual $r$ we get
    \[
    Q(i) = \frac{1}{n}||r||_2^2 + \frac{{\theta_s^i}^2}{n}||b_s^i||_2^2 + \frac{2 \theta_s^i}{n}r^Tb_s^i
    \]
    
    However, since 
    $r$ was the residual of the ordinary least squares when $b_s^i$ was part of the basis functions. Hence $r \perp b_s^i$ for all columns $i$. Also the first term is common for all $i$. Hence the problem of minimizing $Q(i)$ can be replaced by an equivalent problem, proving the presented result
    \[
    j = \arg \min_i Q(i) = \arg \min_i |\theta_s^i|\cdot||b_s^i||_2
    \]
    \qed
\end{proof}

 Line 2 in Algorithm \ref{Algo4} use Lemma \ref{backdel_th0} to choose the least important column (function centered at $x_j$) for deletion. We then define a temporary basis set, coordinate set, and sparse representation ($B_{s0},\Theta_{s0}, C_{s0}$) obtained after removal of the chosen column. This is achieved through the sub-routine $Update\_del(\cdot)$. Here the temporary weights for the updated basis set ($B_{s0}$) can either be computed using (\ref{theta_s}) and (\ref{theta_s2}) and solving in reverse for $(B_{s0}^TB_{s0})^{-1}$ for removal of a column (which would first involve returning $(B_s^TB_s)^{-1}$ from Algorithm \ref{Algo3} for full basis $B_s$ and using it compute our way backwards) or we can directly solve the least squares problem associated with the new basis $B_{s0}$. With this information, we update the residual, and if the total increase in MSE (since the start of the deletion procedure) is bounded above by $\vartheta^2_s \epsilon_s^2/n$ (justified in Remark \ref{rem_set_epsilon}.2), then the deletion procedure is continued, else the current quantities $B_s$, $C_s$ and $\Theta_s$ are returned. This is a rather strict threshold criterion as it places a bound on the cumulative increase in MSE. We also discuss an alternate in Proposition \ref{th_backdel} which considers a threshold with respect to each column removal individually.

\section{Approximation Analysis}

In this section we will analyze the properties of the proposed multiscale algorithm. Here we begin by summarizing a few necessary details
\begin{enumerate}
    \item As  mentioned in the introduction, we assume the $(x_i,y_i) \in D$ are $i.i.d$ samples from some underlying distribution $P(x,y)$. More specifically, we assume $\{y_i\}_{i=1}^n$ to be i.i.d Gaussians. Hence, there exists $\sigma \geq 0$ such that $\forall i$ and $\forall t \in \R$
    \[
    \mathbb{E}_{y_i}e^{(y_i - \mathbb{E}y_i)} \leq e^{\sigma^2 t^2/2}
    \]
    This result comes from the definition of sub-Gaussian random variables \cite{zhang2011adaptive}, which contain Gaussian random variables as a special case.
    \item At any scale s, we assume $\alpha_s \in \R^n$ to be the set of weights that provides best possible approximation to y (in terms of approximation accuracy). Hence $Support(\alpha_s) = \{j: \alpha_s(j) \neq 0\}$ (denoted as $S(\alpha_s)$) can be computed by algorithms such as pivoted-QR decomposition (if needed). Therefore, the set $\{b^j_s: \alpha_s(j) \neq 0\}$ obtained from such an algorithm, forms a basis for the approximation space at scale s.
    \item Analogous to $\alpha_s$, we use $\beta_s \in \R^n$ to denote an alternate representation of model weights ($\Theta_s
    $) computed by the multiscale approach. In essence, $\beta_s$, consists of the same weights as in $\Theta_s$, however, it also contains 0 in place of weights for functions that were not chosen to be a part of $B_s$ (justifying the dimensionality of $\beta_s$ as n). 
\end{enumerate}

Here $\alpha_s$ and $\beta_s$ are defined for the purpose of algorithmic analysis. In practice, we continue to work with $\Theta_s$ as demonstrated in the previous section. Now, we begin with the single scale analysis of our proposed approach. This is followed by a detailed multiscale analysis before we move onto the results section.

\subsection{Single scale analysis}

In this subsection, we analyze the behavior of the approach at a particular scale. Let $\alpha_s, \beta_s$  and $S(\alpha_s),  S(\beta_s)$ be defined as before. With such an algorithmic setting, we firstly consider the forward selection part of the algorithm. In such a setting, when a function is added to $B_s$, the corresponding 0 entry in $\beta_s$ is replaced by the new weight (the previous weights in $\beta_s$ are also updated, using (\ref{theta_s})). Next we define a quantity to estimate the quality of the numerical conditioning for the basis set. This plays a crucial rule in analyzing the performance of the approach. Similar quantities for analysis have also appeared in \cite{tropp2004greed,zhang2009consistency}.

\begin{definition}
\label{def1}
At any scale s, for a current set of weights $\beta_s$ and the corresponding basis set $\Pi$ (obtained by restricting $K_s$ to $S(\beta_s)$), we define \textit{Independence Quotient} $\mu_{\beta_s}$ for a candidate column vector $b^j_s$ to be added next, as follows
\begin{equation}
    \mu_{\beta_s}(b^j_s) = ||(I-P_{\beta_s})b^j_s||_2 = ||(I-\Pi(\Pi^T\Pi)^{-1}\Pi^T)b^j_s||_2
\end{equation}
\end{definition}

$\mu_{\beta_s}(b^j_s)$ quantifies the amount of residual energy which $b^j_s$ has, with respect to the space spanned by the current basis set $\Pi$. $\mu_{\beta_s}$ plays a crucial role of ensuring that all added functions are sufficiently independent and hence the weight computation by ordinary least squares is numerically stable. Considering the given setting, we present our first main result here which provides a recommendation for updating the algorithmic hyperparameter $\epsilon_s$.

\begin{theorem}
\label{set_epsilon}
For forward selection, choosing the hyperparameter ($\epsilon_s$) in Algorithm \ref{Algo1}  as
\begin{equation}
\label{epsilonlim}
    \epsilon_s \geq  \max\Bigg\{\frac{\gamma ||t_s||_2}{\vartheta_s^2} , \frac{\sqrt{n \Delta}}{\vartheta_s} \Bigg\}
\end{equation}
guarantees
\begin{equation}
\label{cond1}
    \mu_{\beta_s^q}(b^j_s) \geq \gamma, \quad (\gamma > 0)
\end{equation}
\begin{equation}
\label{cond2}
    \mathcal{E}(\beta_s^q) - \mathcal{E}(\beta_s^{q+1}) \geq \Delta, \quad (\Delta > 0),
\end{equation}
thereby ensuring that at any iteration q of forward selection, the added column is sufficiently independent from the current bases (\ref{cond1}) and the corresponding error reduction is also lower bounded (\ref{cond2}). Here $\mathcal{E}(\beta^q_s)$ is the MSE with weights $\beta_s^q$ at iteration $q$.
\end{theorem}

\begin{proof}
    We begin by pointing out that in the forward selection procedure (Algorithm \ref{Algo3}), an added column has to satisfy the criterion of maximum reduction of the current residual. Hence with the current target $t_s$ and already computed weights $\beta^q_s$ till iteration $q$ of Algorithm \ref{Algo3}, we have (with some weight g, that will be optimized later and $e_i = [0,0,..,1,..0]\in \R^n$ consisting of 0s, with 1 placed at the $i^{th}$ location)
    \begin{align*}
        \inf_i \mathcal{E}(\beta^q_s + ge_i) &= 
        \inf_i \frac{1}{n}||t_s - (K_s \beta^q_s + gb^i_s)||_2^2 = \inf_i \frac{1}{n}||(t_s - K_s \beta^q_s) - g b^i_s||_2^2\\
        &=  \mathcal{E}(\beta^q_s) +  \inf_i \Big\{ \frac{1}{n}||g b^i_s||_2^2 - 2\frac{g}{n}(t_s - K_s \beta^q_s)^Tb^i_s \Big\}
    \end{align*}
    On differentiating with respect to g and setting it to 0 we get
    \[
    g^{*} = \frac{(t_s - K_s \beta^q_s)^Tb^i_s}{||b^i_s||_2^2}
    \]
    Thus, we have
    \begin{align}
        \inf_i \mathcal{E}(\beta^q_s + g^{*}e_i) &= \mathcal{E}(\beta^q_s) + \inf_i \Bigg\{- \frac{|(t_s - K_s \beta^q_s)^T b^i_s|^2}{n ||b^i_s||_2^2} \Bigg\} 
    \end{align}
    Now, since for a chosen column $b^j_s$, while computing the optimal coordinated $\beta_s^{q+1}$, the MSE would be even smaller, i.e. $\mathcal{E}(\beta_s^{q+1}) \leq \mathcal{E}(\beta^q_s + g^{*}e_j)$. Hence we have
    \begin{align}
    \label{increase}
      \mathcal{E}(\beta_s^q) - \mathcal{E}(\beta_s^{q+1}) &\geq  \frac{|(t_s - K_s \beta^q_s)^Tb^j_s|^2}{n ||b^j_s||_2^2} = \frac{||b^j_s||_2^2|(t_s - K_s \beta^q_s)^Tb^j_s|^2}{n ||b^j_s||_2^4}\\
      &\geq \frac{||b^j_s||^2_2 \epsilon_s^2}{n} \geq
      \frac{\vartheta_s^2 \epsilon_s^2}{n} \geq
      \Delta
    \end{align}
    The  first inequality in the second step here comes from the necessary condition of the forward selection algorithm (Algorithm \ref{Algo3}). The second inequality in this step comes from definition of $\vartheta_s$ (\ref{vartheta}). The last inequality is for ensuring a lower bound ($\Delta$) on the error reduction. Hence the error reduction will always be lower bounded if we set $\epsilon_s \geq \sqrt{n \Delta} /\vartheta_s$.
    
    Now, considering the case for sufficient independence of $b^j_s$, with respect to the basis set obtained by restricting the columns of $K_s$ to $S(\beta^q_s)$ (we denote it as $\Pi$). For $b^j_s$ to be selected
    
    \begin{align*}
        \epsilon_s||b^j_s||_2^2 & \leq |(t_s - K_s \beta_s)^T b^j_s| = |(t_s - \Pi (\Pi^T \Pi)^{-1} \Pi^T t_s)^Tb^j_s )|\\
        & = |t_s^T(I - \Pi (\Pi^T \Pi)^{-1} \Pi^T ) b^j_s |\\
        & \leq ||t_s||_2 \cdot \mu_{\beta^q_s}(b^j_s), \quad (\mu_{\beta_s^q}\ \text{is the independence quotient from Def. \ref{def1}})
    \end{align*}
    Hence, we have 
    \[
    \mu_{\beta^q_s}(b^j_s) \geq \frac{\epsilon_s||b^j_s||_2^2}{||t_s||_2} \geq \frac{\epsilon_s \vartheta_s^2}{||t_s||_2} \geq \gamma
    \]
    The last inequality here ensures sufficient independence of the added column. Hence we have $\epsilon_s \geq \gamma||t_s||_2/\vartheta_s^2$. Combining this with the earlier lower bound on $\epsilon_s$ completes the proof.
    \qed
\end{proof}

\begin{remark}
\label{rem_set_epsilon}
Here we summarize a set of remarks, with respect to the result in Theorem \ref{set_epsilon}:
\begin{enumerate}
    \item Since, at every scale we wish to select as many `good' functions as possible. Hence, $\epsilon_s$ needs to be set to the smallest value possible. Using this fact, combined with the result from Theorem \ref{set_epsilon}, we update $\epsilon_s$ in Algorithm \ref{Algo1} as
    \begin{equation}
    \label{setepsilon1}
        \epsilon_s =  \max\Bigg\{\frac{\gamma ||t_s||_2}{\vartheta_s^2} , \frac{\sqrt{n \Delta}}{\vartheta_s} \Bigg\}
    \end{equation}
    Hence for a given $\epsilon_0$ (decided based on Proposition \ref{epsilonrate}), we compute $\gamma$ and $\Delta$ by assuming $\epsilon_0 = \gamma ||y||_2^2/\vartheta_s^2 = \sqrt{n \Delta}/ \vartheta_s$. These supporting parameters are then used to update $\epsilon_s$ at subsequent scales (Algorithm \ref{Algo1}).
    \item Using (\ref{increase}), it is guaranteed that at every step of forward selection, the decrease in MSE is lower bounded in the sense
    \begin{equation}
    \label{back_thresh}
    \mathcal{E}(\beta_s^q) - \mathcal{E}(\beta_s^{q+1}) \geq \frac{\vartheta_s^2\epsilon_s^2}{n}  
    \end{equation}
    By using the right hand term in (\ref{back_thresh}) as an upper limit for total increase in MSE during removal of columns from $B_s$ in backward deletion (Algorithm \ref{Algo4}), we ensure that this increase is `tolerable'.
 \end{enumerate}
\end{remark}

Now, we mention, two important results from literature that will help in our analysis of the algorithmic behavior

\begin{proposition}
\label{bermanis1}
(Proposition 3.7 in \cite{bermanis2013multiscale})Let $X=\{x_i\}_{i=1}^n \subset \R^d$ be a set bounded by a box $B=I_1\times I_2 \times \cdot \cdot \times I_d$, where $I_1,I_2,...,I_d$ are intervals in $\R$ ,and let $K_s$ be the associated Gaussian kernel matrix as in (\ref{K^sT}).Then, if we define numerical rank of $K_s$ with precision $\delta_0$ as
\[
R_{\delta_0}(K_s) = \# \Bigg\{j: \frac{\sigma_j(K_s)}{\sigma_0(K_s)} \geq \delta_0 \Bigg\}
\]
where $\sigma_j(K_s)$ is the $j^{th}$ largest singular value of matrix $K_s$.Then, we have the bound
\begin{equation}
    R_{\delta_0}(K_s) \leq \prod_{i=1}^d \Bigg( \frac{2 |I_i|}{\pi} \sqrt{\kappa_s^{-1}ln( \delta_0^{-1})} + 1 \Bigg)
\end{equation}
Here $|I_i|$ denotes the length  of the interval $I_i$ and $\kappa_s$ is the length scale parameter (\ref{K^sT}).
\end{proposition}

Here Proposition \ref{bermanis1} will provide us a numerical bound on the number of independent functions in kernel $K_s$. One other result which will be crucial to our analysis comes from the behavior of the sub-gaussian random variables and we state it here for completeness

\begin{lemma}
\label{zhangfoba}
(Lemma C.1 in \cite{zhang2011adaptive}) Consider n independent random variables $y_1, y_2,...,y_n$ such that $\mathbb{E}e^{t(y_i - \mathbb{E}y_i)} \leq e^{\sigma^2 t^2/2}\ \forall t$ and $i$. Consider vectors $g_j = [g_{1,j}, g_{2,j},...,g_{n,j}] \in \R^n$ for $j = 1,..,m$, we have for all $\eta \in (0,1)$, with probability larger than $1 - \eta$
\begin{equation}
    \sup_j |g_j^T(y - \mathbb{E}y)| \leq a\sqrt{2 ln(2m/\eta)}
\end{equation}
where $a = \sigma \sup_j||g_j||_2$.
\end{lemma}

Equipped with these results, we move on to introduce the underlying statistical model at each scale. As before, we assume $t_s$ to be the target at a scale (equal to the residual from the previous scale). At each scale we now fit the model

\begin{equation}
\label{model1}
    t_s(x) = f_s(x) + \varepsilon_s, \quad \varepsilon_s \sim \mathcal{N}(0,\sigma_s^2)
\end{equation}

As introduced before, assuming at any scale s , there exists some optimal set of weights ($\alpha_s$) that model $f_s$ ($f_s = K_s \alpha_s$). Hence we have $\mathbb{E}_s[t_s] = K_s \alpha_s$. For example, one possible way to obtain $\alpha_s$ is using pivoted-QR decomposition to obtain a subset of the columns in $K_s$ that form a well conditioned basis for the span of all kernel functions, and then solving for optimal projection on these columns. Although this set won't necessarily be the minimum set that can produce $f_s$. However,  it will produce the best possible approximation at scale s. Like before, $S(\alpha_s) = \{j: \alpha_s(j) \neq 0\}$ denotes the support for such weights.  Using these quantities we now give our second main result that extends our convergence condition at a particular scale ($|(t_s - K_s \beta_s)^Tb^j_s|/||b^j_s||_2^2  < \epsilon_s$ for a chosen column $b^j_s$) to a corresponding result that takes into account that we are modeling a noisy version of the underlying function, and not $\mathbb{E}_s[t_s]$ directly.

\begin{theorem}
\label{Theorem_convergence}
Let $P_{\beta_s}$ represent the projection operator on $S(\beta_s)$ and $\mathbb{E}_s[t_s]$ be the optimal prediction at scale s ($\mathbb{E}_s[t_s] = K_s \alpha_s$). Using these quantities, when the forward selection algorithm stops, with any $\nu \geq \epsilon_s$ we have
\begin{equation}
\label{Theorem_con_eq}
    \mathcal{P} \Bigg( \sup_j \frac{\Big|(\mathbb{E}_s[t_s] - P_{\beta_s}(\mathbb{E}_s[t_s]))^T b^j_s\Big|}{||b^j_s||_2^2} \geq \nu \Bigg) \leq 2|C_s|\exp^{-{\frac{\vartheta_s^4(\nu - \epsilon_s)^2}{2 \sigma_s^2 \tau^{2}_{\beta_s}}}}
\end{equation}
where $\sigma_s^2$ is the variance term from model (\ref{model1}), $\tau_{\beta_s} = \sup_j\mu_{\beta_s}(b^j_s)$, and $|C_s|$ is the size of scale specific sparse representation at the termination of forward selection at scale s.
\end{theorem}

\begin{proof}
We start with the fact that when, the forward selection algorithm stops, $\forall j$ we have $|(t_s - K_s \beta_s)^T b^j_s|/||b^j_s||_2^2 < \epsilon_s$. Hence we have
\begin{align*}
    \epsilon_s ||b^j_s||_2^2 & > |(t_s - K_s\beta_s)^T b^j_s| = |((I - P_{\beta_s})t_s)^T b^j_s|\\
    &= |((I - P_{\beta_s})t_s)^T b^j_s - ((I - P_{\beta_s})\mathbb{E}_s[t_s])^T b^j_s + ((I - P_{\beta_s})\mathbb{E}_s[t_s])^T b^j_s|\\
    &= |(t_s - \mathbb{E}_s[t_s])^T(I - P_{\beta_s})b^j_s + ((I - P_{\beta_s})\mathbb{E}_s[t_s])^T b^j_s|\\
    &\geq |((I - P_{\beta_s})\mathbb{E}_s[t_s])^T b^j_s| - |(t_s - \mathbb{E}_s[t_s])^T(I - P_{\beta_s})b^j_s|
\end{align*}
Hence we have
\begin{equation}
\label{p1}
    |((I - P_{\beta_s})\mathbb{E}_s[t_s])^T b^j_s| < ||b^j_s||_2^2\epsilon_s + |(t_s - \mathbb{E}_s[t_s])^T(I - P_{\beta_s})b^j_s|
\end{equation}
Considering the second term in (\ref{p1}), we now use Lemma \ref{zhangfoba}. Here we are able to use this result, because the initial Gaussian assumption on the data carries forward to subsequent scales (with different model parameters but still being Gaussian) through the model (\ref{model1}). We firstly compute the term `a' from Lemma \ref{zhangfoba}
\[
a = \sqrt{\sigma_s^2 \sup_j {b^j_s}^T(I - P_{\beta_s})^T (I - P_{\beta_s})b^j_s} = \sigma_s \sup_j||(I - P_{\beta_s})b^j_s||_2
\]
From Def. 3, we have $\tau_{\beta_s} = \sup_j \mu_{\beta_s}(b^j_s) = \sup_j||(I - P_{\beta_s})b^j_s||_2$ giving us $a = \sigma_s \tau_{\beta_s}$. Hence, from Lemma \ref{zhangfoba} with probability greater than $1 - \eta$ where $\eta \in (0,1)$, we have
\[
|(t_s - \mathbb{E}_s[t_s])^T(I - P_{\beta_s})b^j_s| \leq \sqrt{2 \sigma_s^2 {\tau^2_{\beta_s}}\ln \Bigg(\frac{2|C_s|}{\eta} \Bigg)}
\]
Hence from (\ref{p1}) and using the definition of $\vartheta_s$ (\ref{vartheta}), we have with probability greater than $1 - \eta$, $\forall\ j$

\begin{equation}
\label{p2}
    \frac{|((I - P_{\beta_s})\mathbb{E}_s[t_s])^T b^j_s|}{||b^j_s||_2^2} < \epsilon_s + \sqrt{\frac{2 \sigma_s^2 \tau^2_{\beta_s}}{\vartheta_s^4}\ln \Bigg(\frac{2|C_s|}{\eta} \Bigg)}
\end{equation}

Now let $\nu = \epsilon_s + \sqrt{\frac{2 \sigma_s^2 \tau^2_{\beta_s}}{\vartheta_s^4}\ln \Big(\frac{2|C_s|}{\eta} \Big)}$ (hence $\nu \geq \epsilon_s$). With solving and rearranging we get
\[
\eta = 2|C_s|\exp^{-\frac{\vartheta_{s}^4 (\nu - \epsilon_s)^2}{2 \sigma_s^2 \tau^2_{\beta_s}}}
\]
(\ref{p2}) can now be equivalently expressed as

\[
P\Bigg(\sup_j \frac{|((I - P_{\beta_s})\mathbb{E}_s[t_s])^Tb^j_s|}{||b^j_s||_2^2} \geq \epsilon_s + \sqrt{\frac{2 \sigma_s^2 \tau^2_{\beta_s}}{\vartheta^4_s}\ln \Bigg(\frac{2|C_s|}{\eta} \Bigg)} \Bigg) \leq \eta
\]
Substituting $\eta$ and $\nu$ completes the proof.
\qed
\end{proof}

Before moving to the next part of the analysis, here we mention a few remarks demonstrating the relevance of the result presented in Theorem \ref{Theorem_convergence}.

\begin{remark} The result in Theorem \ref{Theorem_convergence} can in inferred in multiple ways. Here we summarize them
\begin{enumerate}
    \item If we only consider the upper bound in Theorem \ref{Theorem_convergence} to be useful when it is smaller than some function of $\nu$: $0 < h(\nu) < 1$. Then we have
    \[
    2|C_s|\exp^{-{\frac{\vartheta_s^4(\nu - \epsilon_s)^2}{2 \sigma_s^2 \tau^{2}_{\beta_s}}}} \leq h(\nu)
    \]
    Hence, the probabilistic bound for `convergence' at a scale (here convergence refers to the ability of capturing scale specific information in data efficiently) proved in Theorem \ref{Theorem_convergence} is only useful for small noise levels (obtained by rearranging the above relation)
    \begin{equation}
        \label{noise}
        \sigma_s^2 \leq \frac{\vartheta^4_s(\nu -\epsilon_s)^2}{2 \tau^2_{\beta_s} \ln \Big( \frac{2|C_s|}{h(\nu)} \Big)} \leq \frac{\vartheta^4_s(\nu -\epsilon_s)^2}{2 \tau^2_{\beta_s} \ln \big( \frac{2}{h(\nu)} \big)}
    \end{equation}
    Here the second inequality results from the fact $|C_s| \geq 1$, i.e. the bound makes sense for scales which contribute atleast one function to the final approximation. Hence (from (\ref{noise})), the multiscale approach is guaranteed to be efficient for low noise setting. For guarantees at higher noise levels, additional smoothing regularization might be preferred.  
    \item For the bound in (\ref{Theorem_con_eq}), if $X_s = \{x_j: j \in S(\alpha_s)\}$ (i.e. our algorithm chose the same functions as the support of $\alpha_s$), the LHS goes to zeros as $P_{\beta_s} (\mathbb{E}[t_s]) = \mathbb{E}[t_s]$. However the RHS, also goes to 0 as $\tau^2_{\beta_s} = 0$ (from the definition of $\alpha_s$). Hence the bound is tight in that sense.
    \item A natural bound for the number of elements in $C_s$ is the numerical rank of the kernel $K_s$ at that particular scale. Hence, using Proposition \ref{bermanis1}, we modify the upper bound as
    \begin{align}
        RHS &=2|C_s|\exp^{-{\frac{\vartheta_s^4(\nu - \epsilon_s)^2}{2 \sigma_s^2 \tau^{2}_{\beta_s}}}} \\
        & \leq 2 \Bigg\{ \prod_{i=1}^d \Bigg(\frac{2 |I_i|}{\pi \sqrt{T}} \sqrt{2^sln( \delta_0^{-1})} + 1 \Bigg)\Bigg\} \exp^{-{\frac{\vartheta_s^4(\nu - \epsilon_s)^2}{2 \sigma_s^2 \tau^{2}_{\beta_s}}}}
    \end{align}
    where T is from equation (\ref{T}). This formulation directly provides an upper bounds in terms of the scale number with these Gaussian kernels. Bounding $|C_s|$ with the numerical rank of the $K_s$ is highly efficient at initial scales (because of the availability of few independent functions). However, this bound becomes loose at higher scales.
    \item Coming back to (\ref{p2}) with probability more than $1 - \eta$ we have
    \[
    \frac{|((I - P_{\beta_s})\mathbb{E}_s[t_s])^T b^j_s|}{||b^j_s||_2^2} < \epsilon_s + \sqrt{\frac{2 \sigma_s^2 \tau^2_{\beta_s}}{\vartheta_s^4}\ln \Bigg(\frac{2|C_s|}{\eta} \Bigg)}
    \]
    Here, we can very nicely see the division of overall error into two components
    \begin{itemize}
        \item $\epsilon_s$: The approximation error introduced due to the model
        \item $\sqrt{\frac{2 \sigma_s^2 \tau^2_{\beta_s}}{\vartheta_s^4}\ln \Big(\frac{2|C_s|}{\eta} \Big)}$: Estimation error due to noise.
    \end{itemize}
    \item If a scale is completely ignored by the forward selection algorithm, then we will have $|C_s| = 0$. Hence, even with smallest $\nu ( = \epsilon_s)$, we will have the RHS term 0 (\ref{Theorem_con_eq}). Now since in such a case $P_{\beta_s}(\mathbb{E}_s[t_s]) = 0$ (as $|C_s| = 0$). Hence for $|C_s| = 0$ we have 
    \[
    \sup_j \frac{|\mathbb{E}_s[t_s]^T b^j_s|}{||b^j_s||_2^2} < \epsilon_s
    \]
    
    But $\mathbb{E}_s[t_s]$ has a representation $K_s \alpha_s$. Hence, this dot product is small because $\mathbb{E}_s[t_s]$ itself is very small. This directly leads us to conclude that if a scale is completely ignored by our forward selection algorithm, then that scale is `truly' not useful and its contribution to the final approximation is very small and bounded at best.
\end{enumerate}
\end{remark}

Now, we move on to the backward deletion criterion for removing functions to obtain a minimum set. Since, the tolerance on error increment with removal of functions is fairly user chosen. Hence, besides the \textit{more strict condition} (shown in Algorithm \ref{Algo4}, which removes the functions as long as the cumulative increase in MSE is bounded), in this section we discuss one more alternate condition where the same threshold is placed on the increase in MSE, with each removal of a column. Here we provide a result than ensures that we only remove columns as long as the increase in MSE (per column removal) is less than the minimum reduction in MSE while adding columns in the forward selection algorithm.

\begin{proposition}
\label{th_backdel}
At any iteration $u$ of the backward deletion algorithm, if the column ($b^j_s$) chosen through the result in Lemma \ref{backdel_th0} is removed only when it satisfies
\begin{equation}
    \label{th_backdel_eq}
    \inf_j( |\beta^u_s(j)|\cdot||b^j_s||_2) \leq  \vartheta_s \epsilon_s,
\end{equation}
then we have the guarantee that the increase in MSE at that iteration of the algorithm ($\deltaback_{backward}(\mathcal{E}(\beta^u_s))$)
is upper bounded in the sense
 \[
 \deltaback_{backward}(\mathcal{E}(\beta^u_s)) \leq \min_{q}\deltaforward_{forward}(\mathcal{E}(\beta^q_s))
 \]
 where the right hand side term shows the minimum reduction of MSE at any iteration $q$ of forward selection at the same scale.
\end{proposition}

\begin{proof}
    Since, we remove the column that minimizes the resultant MSE (as in Lemma \ref{backdel_th0}), so for this proof we start with the same logic. Assuming right now we are at $u^{th}$ iteration of the backward deletion procedure and considering the removal of the next column 
    \begin{align*}
        \inf_j \mathcal{E}(\beta_s^u - \beta_s^u(j)e_j) &= \inf_j \frac{1}{n}||t_s - (K_s \beta_s^u - \beta_s^u(j)b^j_s)||^2_2\\
        &= \inf_j \frac{1}{n}||t_s - K_s \beta_s^u + \beta_s^u(j) b^j_s||^2_2\\
        &= \mathcal{E}(\beta_s^u) + \inf_j \frac{1}{n}||\beta_s^u(j)b^j_s||^2_2
    \end{align*}
    
    Here the inner product term is absent because $j \in S(\beta^u_s)$ and hence $(t_s - K^s \beta^u_s) \perp b^j_s$. Now since, we know after the removal of $j^{th}$ column we will receive $\beta_s^{u+1}$ and $\mathcal{E}(\beta_s^{u+1}) \leq \inf_j \mathcal{E}(\beta_s^u - \beta_s^u(j)e_j)$. Hence we have
    \begin{equation}
    \label{b}
    \deltaback_{backward}(\mathcal{E}(\beta^u_s)) = 
    \mathcal{E}(\beta_s^{u+1}) - \mathcal{E}(\beta_s^{u}) \leq \inf_j \frac{1}{n}||\beta_s^u(j)b^j_s||^2_2
    \end{equation}
    
    Now from (\ref{increase}) we know that, with any iteration of the forward selection algorithm the reduction in MSE (with addition of any $b^j_s$) is lower bounded as
    
    \begin{equation}
    \label{f}
    \deltaforward_{forward}(\mathcal{E}(\beta^q_s)) = 
    \mathcal{E}(\beta_s^q) - \mathcal{E}(\beta_s^{q+1}) \geq \frac{\vartheta_s^2\epsilon_s^2}{n} 
    \end{equation}
    
    Hence, from (\ref{b}) and (\ref{f}), for $\deltaback_{backward}(\mathcal{E}(\beta^u_s)) \leq \min_{q}\deltaforward_{forward}(\mathcal{E}(\beta^q_s))$, it is sufficient to have
    \[
    \inf_j \frac{1}{n}||\beta_s^u(j)b^j_s||^2_2 \leq \frac{\vartheta_s^2\epsilon_s^2}{n} 
    \]
    which completes the proof of the presented result.
    \qed
\end{proof}

\subsection{Multiscale analysis}

In this section we extend the single scale analysis to the overall algorithm. We start with analyzing the behavior of $\epsilon_s$ which is updated as per Theorem \ref{set_epsilon} and is essentially, the main hyperparameter controlling the performance and coupling across scales.

\begin{proposition}
\label{epsilonrate}
Using the definition of $\epsilon_s$ in (\ref{setepsilon1}) following Theorem \ref{set_epsilon}, we have
\begin{enumerate}
    \item $\epsilon_s$ increases monotonically with s, with a rate of atleast $\vartheta_0/\vartheta_s$.
    \item With increasing scales, $\epsilon_s$ remains bounded in the sense $\epsilon_s < \delta$ (for all $s \in \mathcal{I}$) if $\epsilon_0$ is initialized such that it satisfies $\epsilon_0 < (\delta\vartheta_s)/\vartheta_0$. 
\end{enumerate}
\end{proposition}

\begin{proof}
    Starting with the update criterion for $\epsilon_s$ from Theorem \ref{set_epsilon}
    \[
    \epsilon_s =  \max\Big\{\frac{\gamma ||t_s||_2}{\vartheta_s^2} , \frac{\sqrt{n \Delta}}{\vartheta_s} \Big\}
    \]
    Hence, for a given $\epsilon_0$, and assuming the base case $\epsilon_0 =  \frac{\gamma ||t_0||_2}{\vartheta_0^2} = \frac{\sqrt{n \Delta}}{\vartheta_0}$  we obtain the value of $\gamma$ and $\Delta$ as: $\gamma = \frac{\epsilon_0 \vartheta_0^2}{||t_0||_2}$ and $\Delta = \frac{\epsilon_0^2 \vartheta_0^2}{n}$. Thus we have,
    \begin{align*}
        \epsilon_s &= \max\Bigg\{\epsilon_0\Bigg(\frac{\vartheta_0}{\vartheta_s}\Bigg)^2  \frac{||t_s||_2}{||t_0||_2}, \epsilon_0 \frac{\vartheta_0}{\vartheta_s}\Bigg\}\\
        &= \epsilon_0 \Bigg(\frac{\vartheta_0}{\vartheta_s} \Bigg) \max \Bigg\{ \frac{\vartheta_0||t_s||_2}{\vartheta_s||t_0||_2},1 \Bigg\}
    \end{align*}
    Since, based on the structure of squared exponential functions, $\vartheta_0 \geq \vartheta_s (\forall\ s \in \mathcal{I})$. Hence the two terms multiplied to $\epsilon_0$ in the above equation are both $\geq 1, \forall s \in \mathcal{I}$. Thus, completing the proof for part 1. We will further analyze the behavior of $\epsilon_s$ in the results section to provide further intuition.
    
    Proof for part 2 of the proposition follows follows directly from the previous expression and using $\epsilon_s < \delta$
    \[
    \epsilon_0 < \frac{\delta}{\Big(\frac{\vartheta_0}{\vartheta_s} \Big) \max \Big\{ \frac{\vartheta_0||t_s||_2}{\vartheta_s||t_0||_2},1 \Big\}} \leq \frac{\delta}{\vartheta_0/\vartheta_s}
    \]
    Hence proving the claimed result.
    \qed
\end{proof}

\begin{remark}
Since $\delta$ in Proposition \ref{epsilonrate} is always less than 1 (for $\epsilon_s < \delta$ condition to be useful). Therefore, $\epsilon_0$ should never be initialized to a value greater than $\vartheta_s/\vartheta_0$. Hence, by picking a fairly high scale which is expected to be greater than the terminal scale $\omega$(for example $s = 15$), we can make an intelligent initialization of $\epsilon_0 = (\vartheta_s \delta)/\vartheta_0$ (or $\vartheta_s/\vartheta_0$ for $\delta = 1$).
\end{remark}
Moving further, we analyze the truncation scale $\omega$, and how the approximation from the truncated kernel behaves with respect to the kernel with infinite scales. Here, we first begin with a simple lemma, where we determine the truncation index based on the conditioning behavior of the truncated kernel $K^{\omega}$.
\begin{lemma}
\label{truncation1}
If the truncation is made at a sufficiently high scale such that the Gramian matrix from the truncated kernel function $K^{\omega}$ is numerically full rank, then the weights for the full kernel function K (denoted as $\phi$) and weights for the truncated kernel function $K^{\omega}$ (denoted as $\phi^{\omega}$) while approximating $y$ are the same, i.e. $\phi = \phi^{\omega}$.
\end{lemma}

\begin{proof}
    Considering an approximation from the full kernel
    \begin{align*}
    Af &= \sum_{i=1}^n \phi_i K(x_i,\cdot) = \sum_{i=1}^n \phi_i \Bigg(\sum_{s = 0}^{\infty} \zeta_s \sum_{x_j \in X_s} \psi_j^s(x_i) \psi_j^s(\cdot) \Bigg)\\
    &= \sum_{s = 0}^{\infty} \sum_{x_j \in X_s} \Bigg(\sum_{i=1}^n \zeta_s \phi_i \psi_j^s(x_i) \Bigg) \psi_j^s(\cdot)\\
    &= \underbrace{ \sum_{s = 0}^{\omega} \sum_{x_j \in X_s} \Big(\sum_{i=1}^n \zeta_s \phi_i \psi_j^s(x_i) \Big) \psi_j^s(\cdot)}_\text{Approximation till scale $\omega$} + \underbrace{ \sum_{s = \omega+1}^{\infty} \sum_{x_j \in X_s} \Big(\sum_{i=1}^n \zeta_s \phi_i \psi_j^s(x_i) \Big) \psi_j^s(\cdot)}_\text{Approximation from higher scales}
    \end{align*}
    Now, since the approximations at scales are greedily computed one by one (with addition of functions $\psi_j^s$ and computation of coefficients $\sum_{i=1}^n \zeta_s \phi_i \psi_j^s(x_s)$) with increase in scales. Hence, the part of approximation till scale $\omega$ will be equal to the approximation produced by the kernel truncated at scale $\omega$. Considering such an approximation
    \[
    Af^{\omega} = \sum_{i=1}^n \phi^{\omega}_i K^{\omega}(x_i,\cdot) = \sum_{s = 0}^{\omega} \sum_{x_j \in X_s} \Big(\sum_{i=1}^n \zeta_s \phi^{\omega}_i \psi_j^s(x_i) \Big) \psi_j^s(\cdot)
    \]
    Now, using the fact $Af|_{\omega} = Af^{\omega}$ we obtain
    \[
    \sum_{s = 0}^{\omega} \sum_{x_j \in X_s} \Big(\sum_{i=1}^n \zeta_s (\phi_i - \phi^{\omega}_i) \psi_j^s(x_i) \Big) \psi_j^s(\cdot) = 0
    \]
    Which leads to
    \[
    \sum_{i=1}^n (\phi_i - \phi_i^{\omega})\sum_{s=0}^{\omega}\zeta_s \sum_{x_j \in X_s}\psi_j^s(x_i) \psi_j^s(\cdot) = \sum_{i=1}^n(\phi_i - \phi_i^{\omega})K^{\omega}(\cdot,x_i) = 0
    \]
    Since, we know that the truncation scale $\omega$ was high enough such that Gramian matrix for $K^{\omega}$ is numerically positive definite (by assumption of the lemma). Hence, only way for this summation to be 0 is if $\phi_i - \phi_i^{\omega} = 0,\ \forall i$. Hence we conclude $\phi = \phi^{\omega}$
    \qed
\end{proof}

Lemma \ref{truncation1} provides a very useful result in the sense that, once our truncation scale is sufficiently high such that the associated Gramian matrix is numerically full rank, then the coefficients for such approximation are universal and can be used with kernels with truncation at any higher scale $s > \omega$. Next we provide a result that uses the size of the sparse set at higher scales as a determining factor for truncation index.  Here our analysis is motivated from \cite{griebel2015multiscale}.

\begin{theorem}
If the multiscale algorithm is truncated at a sufficiently high scale $\omega$, such that for scales higher than $\omega$, we have $\rho = \sum_{s = \omega+1}^{\infty} \zeta_s |C_s| <  \sigma_{min}(K)/{n}$, where $\sigma_{min}(K)$ is the smallest singular value for the Gramian matrix for kernel K. Then, with $Af$ and $Af^{\omega}$ representing the approximation from infinite scale kernel ($K$) and truncated scale kernel ($K^{\omega}$) respectively, we have 
\begin{equation}
    ||Af - Af^{\omega}||_{\mathcal{H}} \leq \sqrt{\frac{{n} \rho}{\sigma_{min}(K)}} ||Af||_{\mathcal{H}}
\end{equation}
\end{theorem}

\begin{proof}
    Directly starting with the error of truncation at some scale $\omega$, let $\phi$ be the weights for approximation with full kernel (correspondingly, $\phi^{\omega}$ be the weight for the truncated kernel). Here, in the second equality we use Lemma \ref{truncation1} to get the expression for component of approximation after truncation. 
    \begin{align*}
        ||Af - Af^{\omega}||^2_{\mathcal{H}} &= ||K \phi - K^{\omega} \phi^{\omega}||^2_{
        \mathcal{H}}\\ &=  \Bigg|\Bigg| \sum_{s = \omega+1}^{\infty} \sum_{x_j \in X_s} \Big(\sum_{i=1}^n \zeta_s \phi_i \psi_j^s(x_i) \Big) \psi_j^s(\cdot) \Bigg|\Bigg|^2_{\mathcal{H}}\\
        &= \sum_{s = \omega + 1}^{\infty} \sum_{x_j \in X_s} \frac{|\sum_{i=1}^n \zeta_s \phi_i \psi_j^s(x_i)|^2}{\zeta_s}\\
        & \leq \sum_{s = \omega + 1}^{\infty} \sum_{x_j \in X_s} \zeta_s \Bigg(\sum_{i=1}^n |\phi_s|\cdot|\psi_j^s(x_i)| \Bigg)^2\\
        &\leq \sum_{s = \omega + 1}^{\infty} \sum_{x_j \in X_s} \zeta_s \Bigg( \sum_{i=1}^n |\phi_i|\Bigg)^2, \quad \text{since }|\psi_j(x_i)| \leq 1\\
        & \leq \sum_{s = \omega + 1}^{\infty} \sum_{x_j \in X_s} n\zeta_s ||\phi||^2_2 = \sum_{s=\omega+1}^{\infty}n \zeta_s |C_s|\cdot ||\phi||_2^2
    \end{align*}
    Here the last inequality uses equality of norms in finite dimensional spaces. Now, for any RKHS (with $\sigma_{min}(K)$ representing the minimum singular value of the Gramian matrix, obtained from the multiscale kernel K), we have
    \[
    ||Af||^2_{\mathcal{H}} = \Bigg<\sum_{i=1}^n\phi_i K(\cdot,x_i),\sum_{j=1}^n\phi_j K(\cdot,x_j) \Bigg>_{\mathcal{H}} \geq \sigma_{min}(K)||\phi||_2^2
    \]
    Hence, now we have
    \[
    ||Af - Af^{\omega}||^2_{\mathcal{H}} \leq \frac{n}{\sigma_{min}(K)}\Bigg(\sum_{s=\omega+1}^{\infty}\zeta_s|C_s|\Bigg) ||Af||^2_{\mathcal{H}}
    \]
    By substituting $\rho$ and taking square root both sides, we obtain the desired result.
    \qed
\end{proof}

Now, we move on to studying the performance of our approach as a multiscale approximation scheme for noisy datasets for which we present the following theorem. Here the idea is to show the boundedness of the obtained sparse representation. Additionally, it also bounds the size of the sparse representation in case the full kernel was used instead of a truncated one.

\begin{theorem}
With truncation at scale $\omega$, when the multiscale algorithm stops, we have
\begin{enumerate}
    \item The total reduction in mean squared error is  at least $\sum_{s = 0}^{\omega} \frac{(|C_s|-1)\vartheta_s^2\epsilon_s^2}{n}$.
    \item The total size of the final sparse representation is upper bounded in the sense
    \begin{equation}
        |C^{\omega}| = \sum_{s = 0}^{\omega}|C_s| \leq \frac{||t_0||_2^2 - ||t_{\omega+1}||_2^2 + \sum_{s=0}^{\omega} \vartheta_s^2 \epsilon_s^2}{\min_{s \leq \omega} \vartheta_s^2 \epsilon_s^2}
    \end{equation}
    where $t_0$ and $t_{\omega+1}$ represent the targets at scale 0 and $\omega+1$ respectively.    \item With finite scale truncation at $\omega$, for $s > \omega$, the number of elements missing from the sparse representation due to truncation is upper bounded in the sense
    \begin{equation}
        |C^{> \omega}| = \sum_{s=\omega+1}^{\infty}|C_s| < \frac{n \mathcal{E}^{\omega}}{\epsilon_{\omega+1}^2} \leq  \frac{\vartheta^4_{\omega+1}}{\gamma^2}
    \end{equation}
    where $\gamma$ is from (\ref{cond1}) and $\mathcal{E}^{\omega}$ is the remaining MSE after truncation scale $\omega$ as shown in Algorithm \ref{Algo1}.
\end{enumerate}

\end{theorem}

\begin{proof}
    (1): Starting with any scale s, since from (\ref{back_thresh}) we know that with addition of every individual column $b^j_s$, the MSE is reduced atleast by $\frac{\vartheta_s^2\epsilon_s^2}{n}$. Hence, for the sparse representation $C^{*}_s$ chosen at this scale during forward selection, the reduction in MSE is atleast
    \[
\frac{|C^{*}_s|\vartheta_s^2\epsilon_s^2}{n}
    \]
    Now, since, in the backward deletion procedure, columns are removed until the cumulative increase in error is less than $\frac{\vartheta_s^2\epsilon_s^2}{n}$. Hence the net reduction in error after scale s is fully explored is lower bounded by
    \begin{equation}
\frac{(|C^{*}_s|-1)\vartheta_s^2\epsilon_s^2}{n}
    \end{equation}
    Using the fact that the final sparse representation at scale s ($C_s$) will follow the relation $|C_s| \leq |C^{*}_s|$, and summing this from scale 0 to $\omega$ completes the proof of (1).
    
    (2): Since the actual total reduction in MSE after truncation at scale $\omega$ (represented as $||y||_2^2/n - \mathcal{E}^{\omega})$ will be lower bounded by the error reduction in part (1). Hence
    \[
    \frac{||y||_2^2}{n} - \mathcal{E}^{\omega} \geq \sum_{s=0}^{\omega} \frac{(|C_s|-1)\vartheta_s^2\epsilon_s^2}{n}
    \]
    On rearranging and using $n\mathcal{E}^{\omega} = ||t_{\omega+1}||_2^2$ and $y = t_0$, we get
    \[
    \sum_{s=0}^{\omega}|C_s|\vartheta_s^2 \epsilon_s^2 \leq \sum_{s=0}^{\omega}\vartheta_s^2 \epsilon_s^2 + ||t_0||_2^2 - ||t_{\omega+1}||_2^2
    \]
    Using the fact $\min_{s \leq \omega} (\vartheta_s^2 \epsilon_s^2) \sum_{s=0}^{\omega} |C_s| \leq \sum_{s=0}^{\omega} \vartheta_s^2 \epsilon_s^2 |C_s|$. we obtain the desired result.

    (3): Now considering the case where we keep on continuing the procedure beyond scale $\omega$. In this case at any scale $s > \omega$, the reduction in MSE with every addition of column for forward selection is still atleast $\frac{\vartheta_s^2\epsilon_s^2}{n}$. Here, we dont consider the increment in MSE due to backward deletion as even in the extreme case, with no backward steps at all scales $>$ $\omega$, the following inequality should be satisfied 

    \[
    \sum_{s=\omega + 1}^{\infty} \frac{|C_s|\vartheta_s^2\epsilon_s^2}{n} \leq \mathcal{E}^{\omega}
    \]
    This inequality follows from the fact that overall reduction in the MSE cannot be more than the total remaining MSE after scale $\omega$. Hence, we have
    \begin{equation}
    \label{sparse_g_omega}
      \sum_{s=\omega + 1}^{\infty} |C_s| \leq \frac{n\mathcal{E}^{\omega}}{\min_{s>\omega} \vartheta_s^2 \epsilon_s^2}       
    \end{equation}
    Now, using the fact that $\vartheta_s$ is a monotonically decreasing function of s (from the definition in (\ref{vartheta})) while satisfying $\vartheta_s \geq 1, \forall s$. Additionally $\epsilon_s$ is monotonically increasing (from Proposition \ref{epsilonrate}). 
    Hence we have
    \[|C^{> \omega}| = 
    \sum_{s=\omega + 1}^{\infty} |C_s| \leq \frac{n\mathcal{E}^{\omega}}{\min_{s>\omega} \vartheta_s^2 \cdot \min_{s>\omega} \epsilon_s^2}  < \frac{n \mathcal{E}^{\omega}}{\epsilon_{\omega+1}^2}
    \]
    
    But $\epsilon_s \geq \frac{\gamma ||t_s||_2}{\vartheta_s^2} = \frac{\gamma \sqrt{n \mathcal{E}^{s-1}} }{\vartheta_s^2}$. Hence $\epsilon_{\omega+1}^2 \geq \frac{\gamma^2 n \mathcal{E}^{\omega}}{\vartheta_{\omega+1}^4}$. Putting it in the above expression gives
    
    \[
    |C^{> \omega}| < \frac{\vartheta^4_{\omega+1}}{\gamma^2}
    \]
    
    Thus completing the proof.
    \qed
\end{proof}

\begin{figure}[h]
\centering
\includegraphics[width=6.5cm]{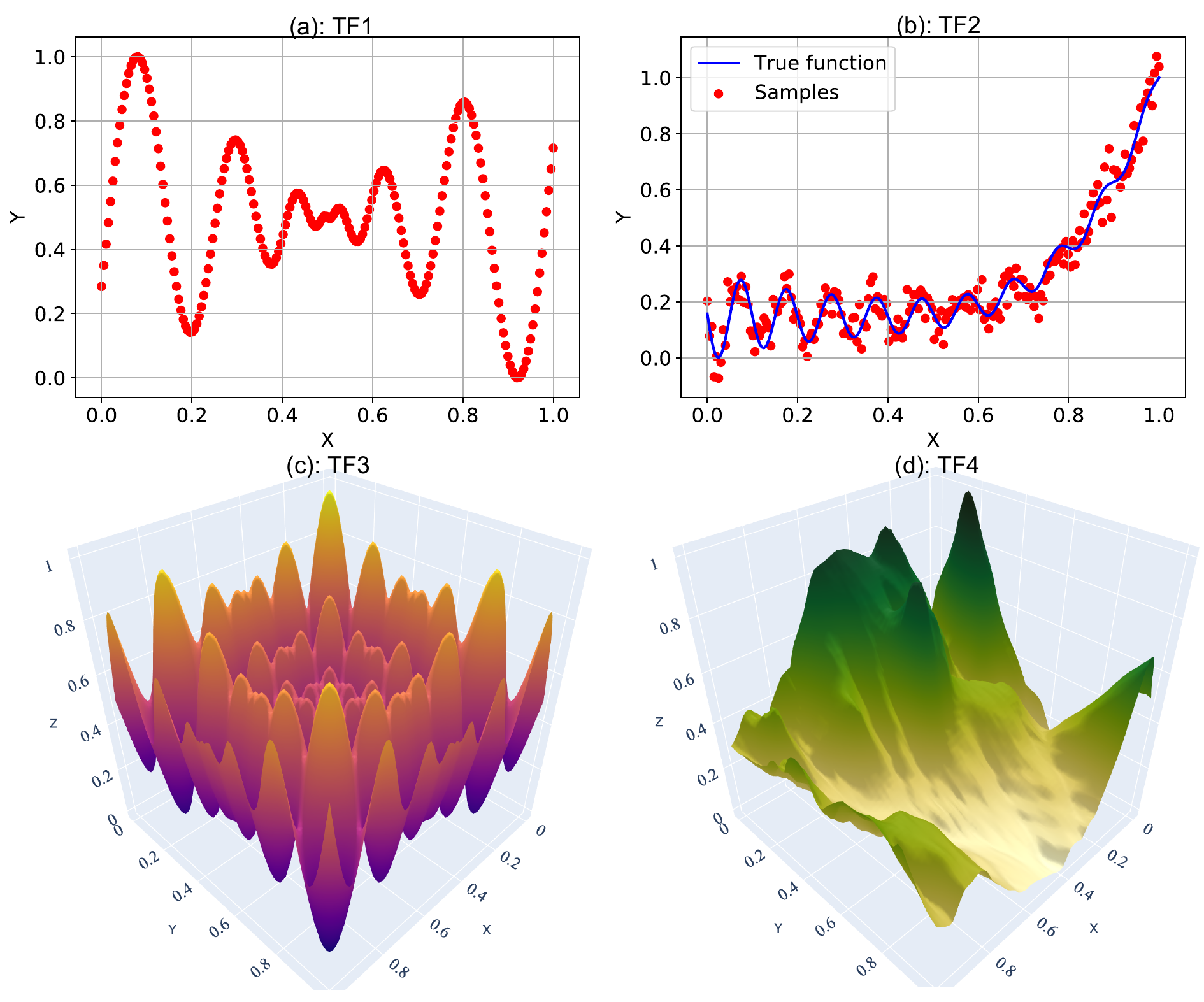}
\caption{Test Functions (a): noise free samples from 1d Schwefel function; (b): noisy samples from 1d Gramacy and Lee function; (c) noise free samples from 2d Schwefel function; (d): a DEM (topography) dataset.}
\label{Pic1}
\end{figure}

\section{Results}
In this section we analyze the performance of the presented multiscale approach on a variety of datasets (code link is in the Declarations section at the end). We begin with the datasets \cite{simulationlib} shown in Fig. \ref{Pic1}. Here samples in (a) and (b) represent univariate datasets (sample size = 200), with (b) specifically infected with noise to represent a more complex scenario. (c) and (d) extends the setting to multivariate datasets. Here (c) is an analytical function with multiple local maxima and minima (sample size = 2500). This is followed by (d), which is a Digital Elevation Model (DEM) dataset (sample size = 5336), i.e. a map of topography. It should be noted that without loss of generality, we have normalized all the datasets to values from 0 to 1. In this section we begin with convergence studies for these four test functions, followed by a study on analyzing the behavior of $\epsilon_s$ and how it is updated over the scales. Convergence is also studied in terms of the relative magnitude of sparse representation (with respect to the full data size). Then we analyze the quality of multiscale reconstructions, more specifically studying the stability of the reconstructions with respect to the sampling design.

\begin{figure}[h]
\centering
\includegraphics[width=10cm]{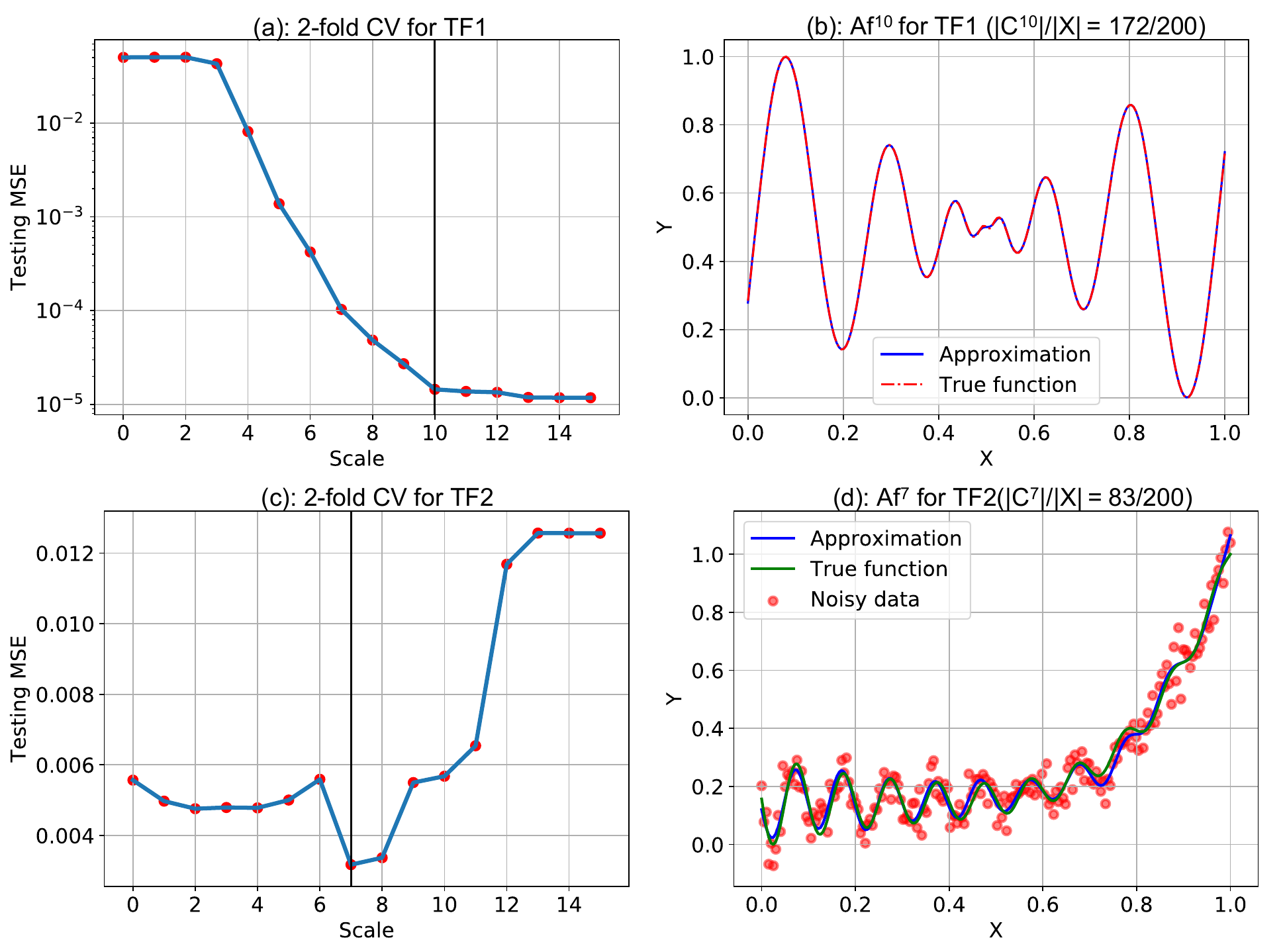}
\caption{Visualizing the decay in testing MSE with increasing scales for TF1 (a) and TF2 (c). For the chosen truncation scales (indicated through black vertical lines), the corresponding reconstructions are shown in (b) and (d) respectively. The headers in (b) and (d) also present the relative sizes of the sparse representation with respect to the total number of samples.}
\label{Pic2}
\end{figure}

We further analyze the performance of our approach as a data reduction and learning model on two other real datasets. The first dataset consists of grayscale images of five different people (comes directly bundled with the sklearn python package \cite{scikit-learn}, though the original data source is  $https://cs.nyu.edu/~roweis/$), where we analyze the ability of data reduction coupled with efficient reconstructions. The last dataset is a photon cloud for the ICESat-2 \cite{neumann2019ice} satellite tracks for McMurdo Dry Valleys in Antarctica. Here the objective is one of the well known problems in Remote Sensing \cite{smith2019land}, dealing with inferring the true surface from large photon clouds.

\begin{figure}[h]
\centering
\includegraphics[width=10cm]{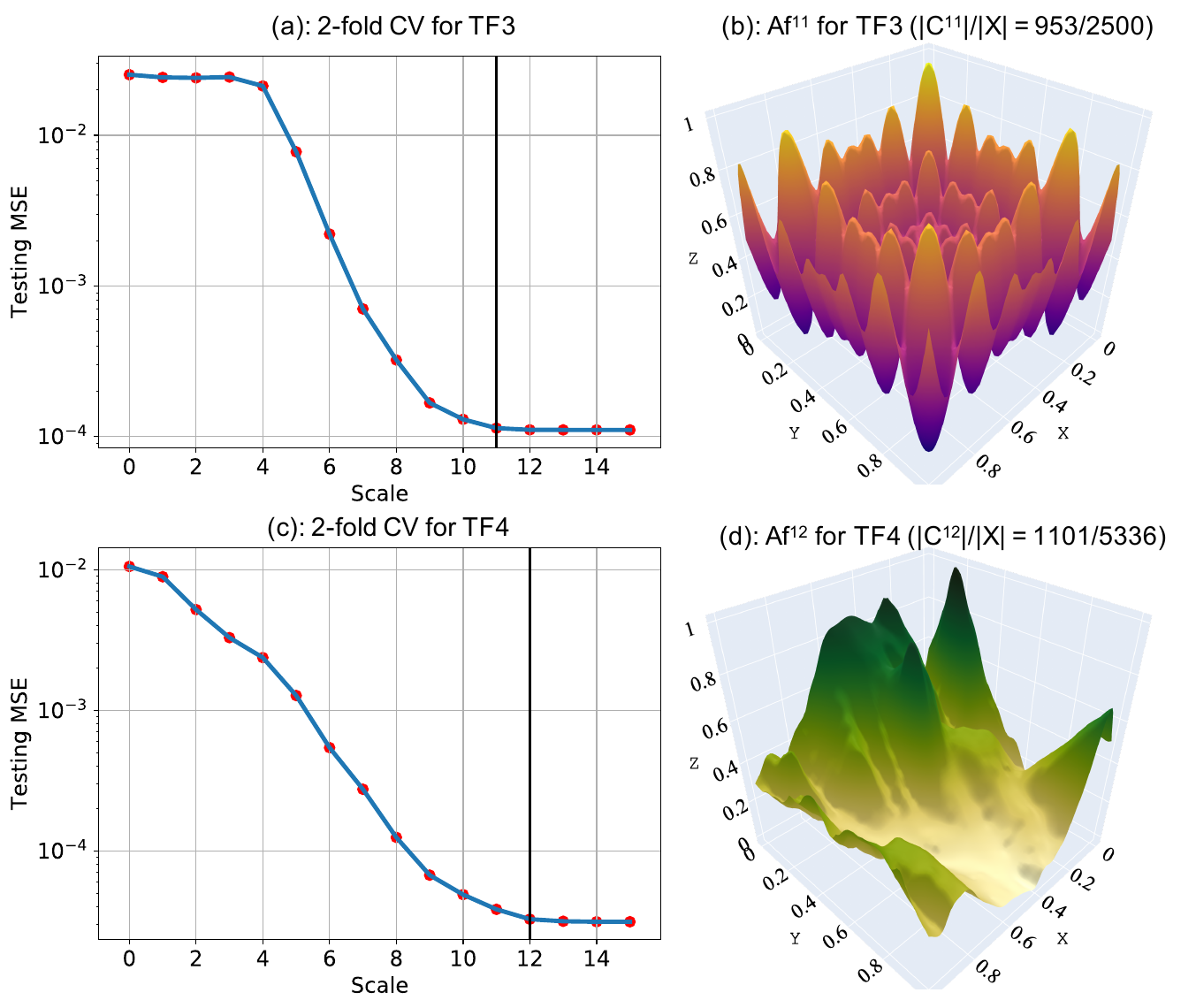}
\caption{Visualizing the decay in testing MSE with increasing scales for TF3 (a) and TF4 (c). For the chosen truncation scales (indicated through black vertical lines), the corresponding reconstructions are shown in (b) and (d) respectively. The headers in (b) and (d) also present the relative sizes of the sparse representation with respect to the total number of samples.}
\label{Pic3}
\end{figure}

Following the results of Proposition \ref{epsilonrate}, and by assuming that we explore the highest scale of 15, we set the starting tolerance $\epsilon_0$ as $\frac{\delta \vartheta_{15}}{\vartheta_0}$. For all univariate approximations we assume $\delta = 10^{-3}$. For higher dimension approximations, we assume $\delta = 10^{-2}$. This assumption is consistent for all the examples shown in this paper. The algorithm is found to work well with other similar small values of $\delta$ as well. 

\subsection{Understanding the general algorithmic behavior}

Following the discussion from the section 3, here we present the full procedure for inferring the truncation scale using K-fold Cross-Validation (K-fold CV). The basic idea is to divide the training data into $K$ parts and using $K-1$ parts for training, while testing on remaining data. This is permuted for the $K$ possible combinations of the training and testing data. The scale that produces minimum mean testing error, is then regarded as our optimal scale. In this paper we present the results for 2-fold CV, which can easily be extended to more splits ($K$) based on user preference.

Figure \ref{Pic2},\ref{Pic3} show the convergence results for the 4 test functions (TF) introduced in Fig. \ref{Pic1}. Starting with Fig. \ref{Pic2}, panel (a) and (c) show the testing MSE for TF1 and TF2, while producing approximations with different truncation scale. The idea here is to find the scale that produces the best generalizable (does well on unseen data points) approximation. Starting with TF1 in panel (a), we have chosen scale 10 to be our optimal scale. It should be noted that while choosing the best scale for a particular dataset is somewhat subjective, the underlying idea is to choose the best trade-off between low-scale and low testing MSE (as truncation at higher scales give better approximation for noiseless datasets, but it also leads to bigger sparse representations, thus diminishing the benefit of data reduction). The approximation with truncation at scale 10 (chosen in Fig. \ref{Pic2}(a)), is shown in Fig. \ref{Pic2}(b). Here $Af^{10}$ uses 172 data points (out of 200) for producing the shown approximation. Moving to TF2 (panel (c) and (d) in Fig. \ref{Pic2}), (c) shows a very different behavior as compared to (a). After reaching the lowest testing MSE at scale 7, it increases to very high levels of prediction errors. This is because of the inherent noise in the samples in TF2, which leads to extreme overfitting at higher scales. The produced approximation $Af^7$ for TF2 is shown in Fig. \ref{Pic2}(d). It should be noted here that for TF2, our sparse representation consists of 83 data points out 200, which is much smaller than TF1. This can be due to the relatively simple functional structure of TF2 as compared to TF1.

\begin{figure}[h]
\centering
\includegraphics[width=12cm]{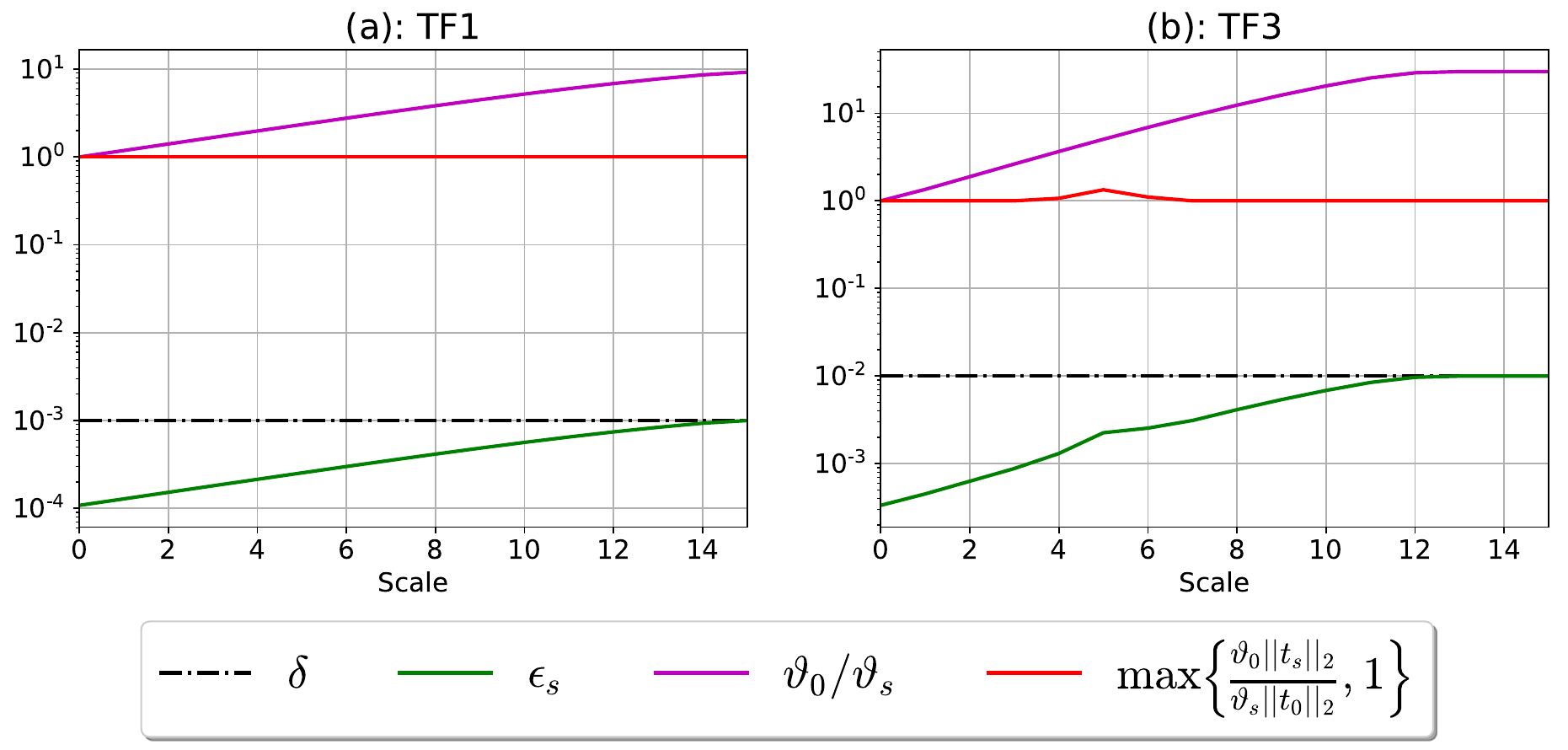}
\caption{Analysis for the behavior of $\epsilon_s$ with increasing scales for TF1 and TF3. This study is targeted to demonstrate results of Proposition \ref{epsilonrate}.}
\label{Pic4}
\end{figure}

Fig. \ref{Pic3} shows the same analysis for TF3 and TF4. Here again it should be noted that we can choose any small scale as long as the testing MSE is acceptable with respect to our application. With convergence scale ($\omega$) as 11 and 12 respectively for TF3 and TF4, we get a data reduction of $62\%(1 - 953/2500)$ and $80\% (1-1101/5336)$ respectively. This higher level of data reduction for TF3 and TF4, show the increased efficiency of the algorithm to produce better sparse representations (due to more redundancy depending on the sampling design) for multivariate datasets.

Next we move onto analyzing the behavior of the tolerance parameter $\epsilon_s$ with increasing scales. As mentioned in Proposition \ref{epsilonrate}, $\epsilon_s$ increases monotonically with s. This helps to avoid algorithmic stalling at higher scales where it is easy to find sufficiently independent functions (due to narrower support of the basis), giving no considerable improvement in the approximation. From Proposition \ref{epsilonrate}, we get
\begin{equation}
\label{pattern}
    \epsilon_s = \epsilon_0 \Bigg(\frac{\vartheta_0}{\vartheta_s} \Bigg) \max \Bigg\{ \frac{\vartheta_0||t_s||_2}{\vartheta_s||t_0||_2},1 \Bigg\}
\end{equation}

Fig. \ref{Pic4} shows the variation of the three different scale dependent terms (which includes $\max \Big\{ \frac{\vartheta_0||t_s||_2}{\vartheta_s||t_0||_2},1 \Big\}$, $\vartheta_0/\vartheta_s$ and $\epsilon_s$) in the above equation. It also confirms the the result of Proposition \ref{epsilonrate} stating $\epsilon_s \leq \delta$ as long as $\epsilon_0$ is set appropriately ($\epsilon_0 \leq \delta \vartheta_s/\vartheta_0$). It is important to note here that among the two terms that affect $\epsilon_s$ (in (\ref{pattern})), $(\vartheta_0/\vartheta_s)$ has a dominating behavior. If we move $\epsilon_0(\vartheta_0/\vartheta_s)$ inside the $max$ function, we realize the condition ensuring a minimum level of improvement ($\Delta$) in the MSE when a new function is added, is generally the driving factor behind this increase in $\epsilon_s$. However, the other condition related to good conditioning of the basis set, also dominates in a few instances depending on how slowly $||t_s||_2$ deteriorates with respect to $||t_0||_2$ (for example, the small red peak at scale 5 in Fig. \ref{Pic4}(b)). Like mentioned before, this overall increase in $\epsilon_s$ is helpful to avoid algorithmic stalling at higher scales resulting from unnecessarily high number of columns chosen during forward selection.

\begin{figure}[h]
\centering
\includegraphics[width=12cm]{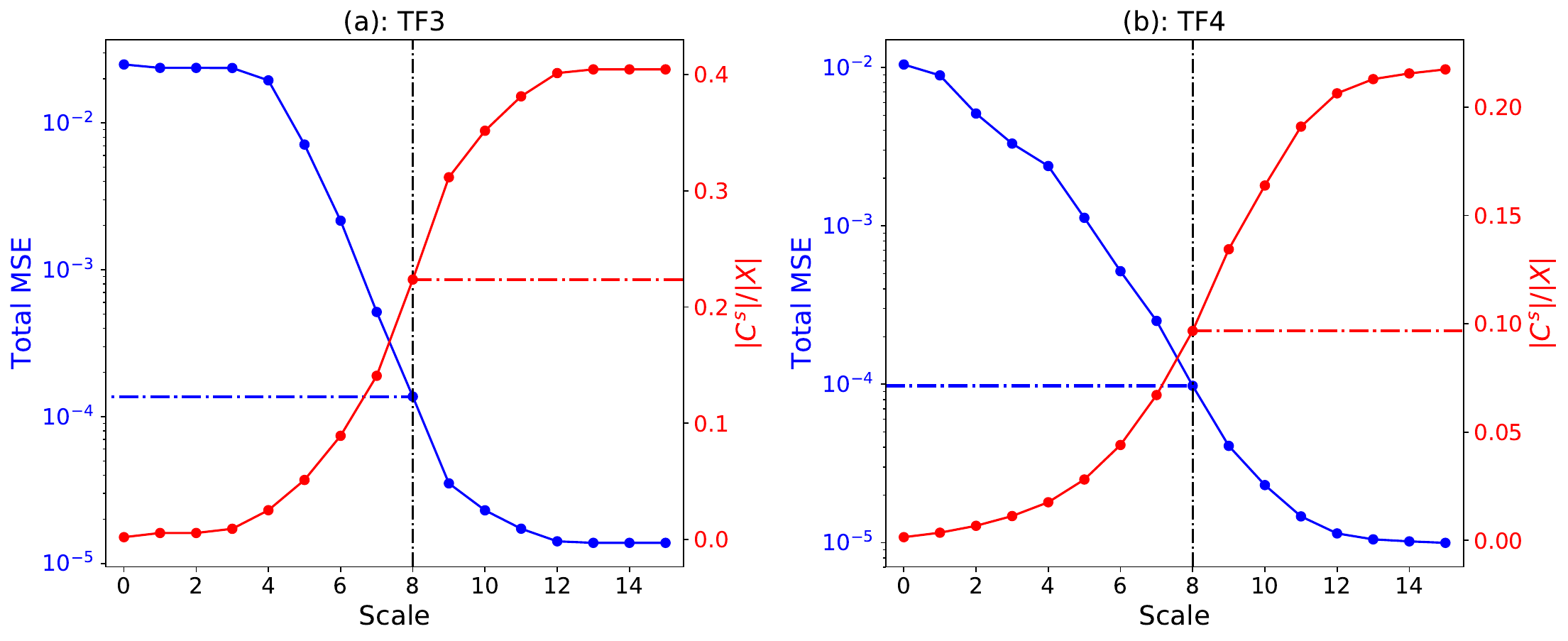}
\caption{The convergence analysis in terms of reduction of total MSE, and relative size of sparse representation with increasing scales for TF3 (a) and TF4 (b). The vertical line at 8 serves as an example for inferring the quality of $Af^8$ in terms of total MSE, while also noting the corresponding relative size ($|C^8|/|X|$) of the sparse representation.}
\label{Pic5}
\end{figure}

\begin{figure}[h]
\centering
\includegraphics[width=12cm]{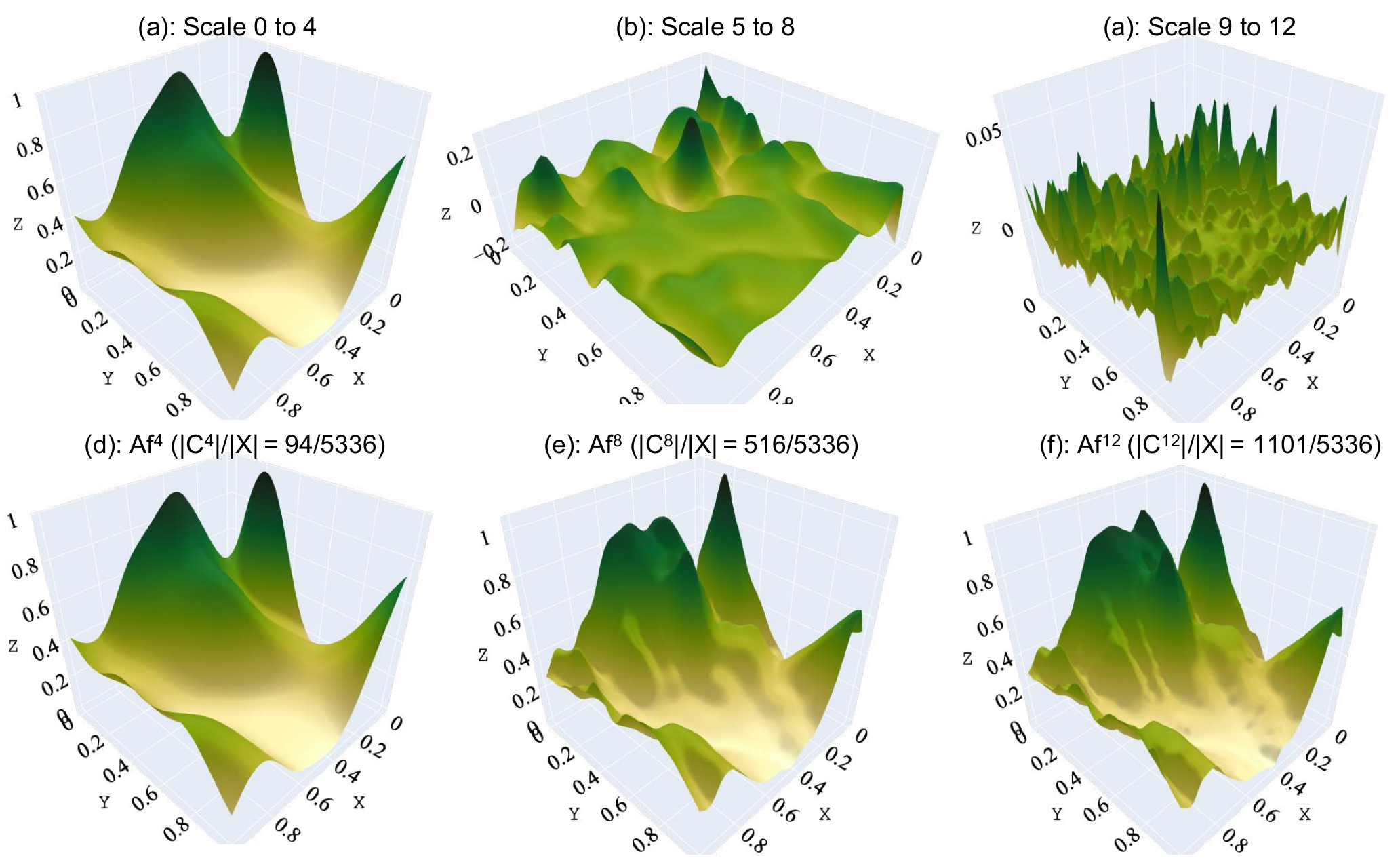}
\caption{Multiscale approximation for TF4. Here (a), (b) and (c) show the contribution of scales 0-4, 5-8 and 9-12 to the final approximation. The plots in the bottom panel ((d), (e) and (f)), show the final approximation obtained till scale 4, 8 and 12 respectively. The headers also show the relative size of the sparse representation to convey the quality of data reduction.}
\label{Pic6}
\end{figure}

While deciding the truncation scale, besides considering the testing MSE, it is also crucial to analyze the size of the respective sparse representation. Fig. \ref{Pic5}, shows the behavior pattern of these two quantities for the 2 bivariate test functions considered (TF3 and TF4). Here we have not shown TF1 and TF2 due to small data sizes for these test functions. Starting from Fig. \ref{Pic5}(a), we analyze the decay of Total MSE (fitting on full data and and computing the reconstruction error), with respect to proportion of sparse representation (denoted as $|C^s|/|X|$). Here by fixing a scale number on $X-axis$ (scale 8 is chosen as an example), we can see where it intersects the blue and red curve, and correspondingly find the reconstruction error and proportion of sparse representation respectively. Using such plots, we can infer results such as, for TF3 at scale 8, with a sparse representation consisting of less than $25\%$ of the full dataset, we can produce approximations with a
total MSE close to $10^{-4}$. Corresponding result for TF4 is shown in panel (b). The data reduction of TF4 is even more efficient, as we also saw in Fig. \ref{Pic3}. Overall, this analysis is really crucial for deciding the truncation scale. Hence, depending on the problem at hand and desired level of data reduction, analysis similar to Fig. \ref{Pic2} and \ref{Pic3} should be coupled with analysis shown in Fig. \ref{Pic5} for an informed decision.

We now move forward to the results in Fig. \ref{Pic6} and analyze the different multiscale approximations produced by the proposed approach. Working with TF4, here we show components contributed by different set of scales in the final overall approximation in panels (a), (b) and (c). Here it should be noted that contributions from scale 0 to 4 almost capture the basic structure (global behavior) of the overall data. The local high frequency components in the data are captured at higher scales (specifically scale 9 to 12 in panel (c)). The bottom panels here show the cumulative approximation till the current scale. Hence (d), (e) and (f) show the final approximations till scale 4, 8 and 12 respectively. It is worth noting here that sharper features are becoming more evident as we move from (d) to (f) in the topographical reconstruction of the full dataset.

\begin{figure}[h]
\centering
\includegraphics[width=10cm]{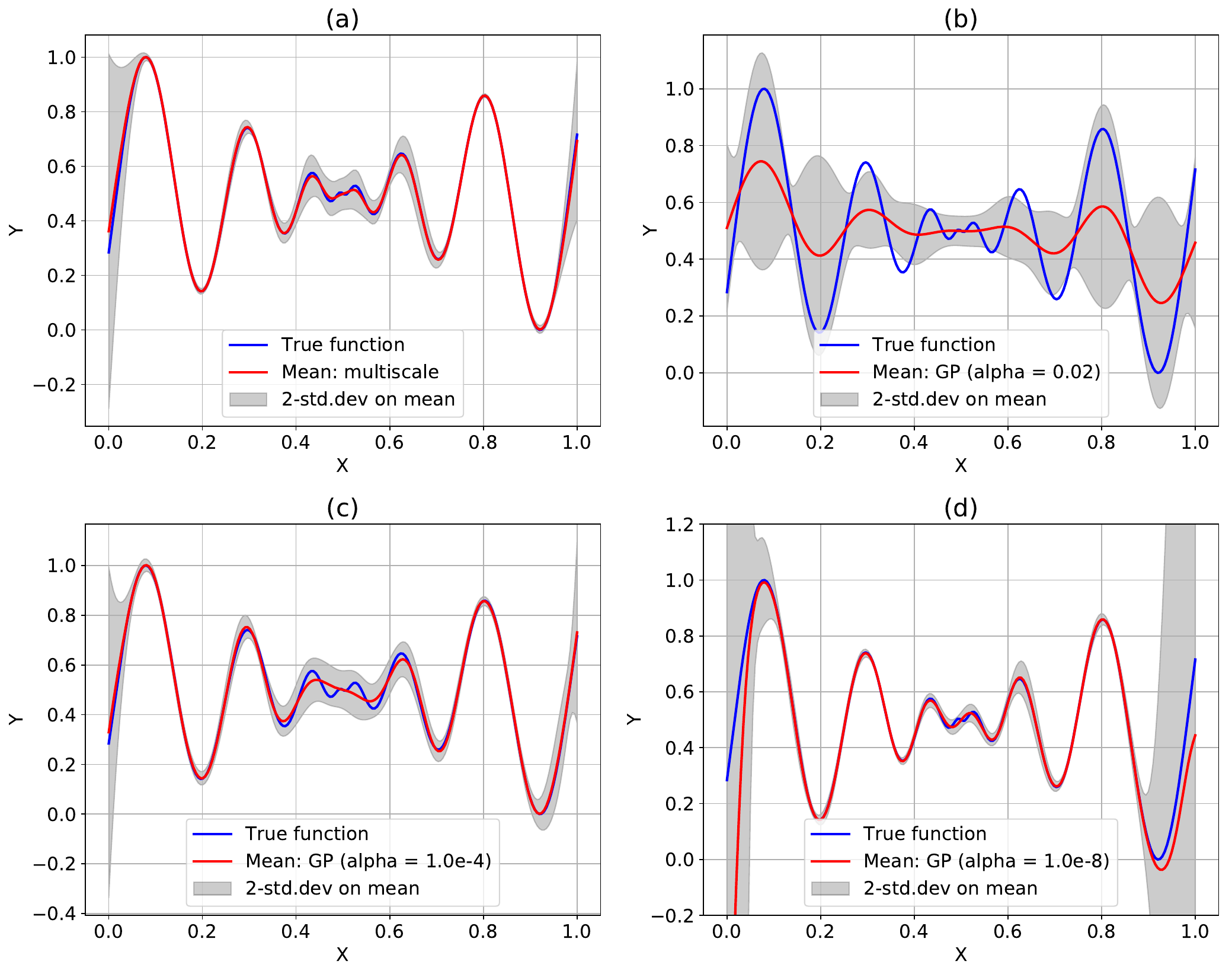}
\caption{Stability study for comparing the performance of the proposed multiscale approach (a), with respect to Gaussian Process regression with three different penalization parameter in (b), (c) and (d). The shown approximations are the mean and 2 standard deviation bounds on the mean for the the approximations produced for 100 different training sets obtained by randomly selecting 50 samples (out of the 200 samples for TF1) without replacement repeatedly for 100 times.}
\label{Pic7}
\end{figure}

\begin{figure}[h]
\centering
\includegraphics[width=12cm]{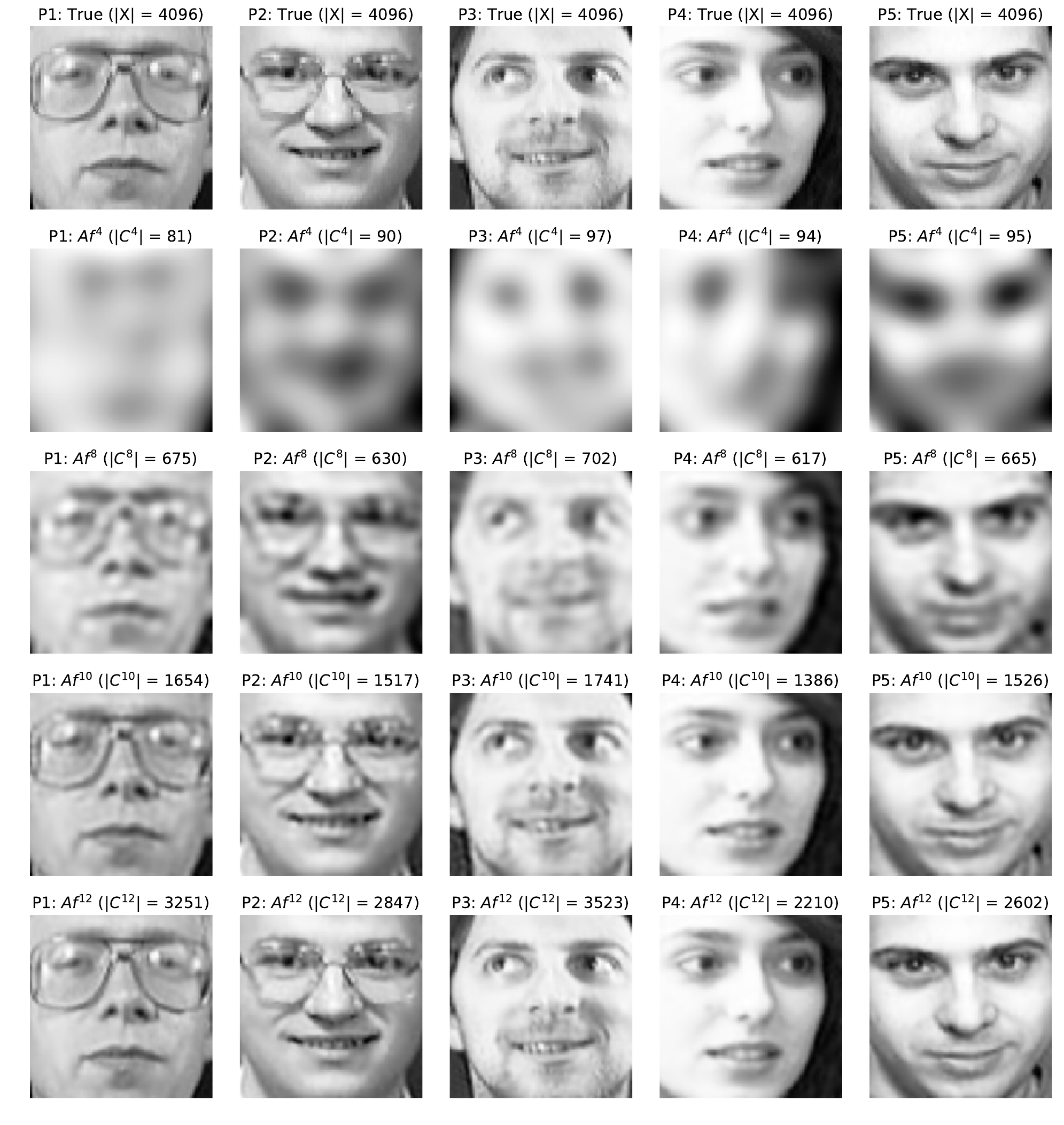}
\caption{\textit{Top row}: 5 grayscale face images ($64 \times 64$) we consider in this section for analysis. \textit{Row 2 to 5}: Each column shows the reconstruction ($100 \times 100$) of the image at scales 4, 8, 10 and 12 respectively for a different person. Besides showing the scale information, the header for each image also shows the corresponding size of sparse representation used to produce the demonstrated  reconstruction.}
\label{Pic8}
\end{figure}

Moving forward, we show a very crucial property of the proposed approach which has to do with stability of the produced approximations with respect to the location of the samples in the training data. For the current study (Fig. \ref{Pic7}), we work with TF1 (univariate functions are easier to visualize). Overall, we randomly select 50 samples from the 200 samples of TF1 without replacement, and carry out this procedure 100 times. This creates a set of 100 training datasets, each having a size of 50. The basic idea is to analyze the produced approximation when samples from different locations are passed as the training data for the same underlying function. In order to keep things in perspective, we have compared the results with Gaussian Processes (sklearn implementation \cite{scikit-learn}) on the same 100 training sets. Fig. \ref{Pic7} shows the mean of the 100 approximations of the function along with 2 standard deviation bounds for quantifying the variation of this mean function and also analyze how far the approximation are from the true function. For Gaussian Processes (GPs), since the dataset doesn't have any noise, we just use the inherent parameter $alpha$ (an added term to the diagonal of the kernel matrix) which should be sufficiently small for not creating artifacts, but also should be large enough to make the computations numerically stable (sklearn documentation \cite{scikit-learn}). Fig. \ref{Pic7} shows results for 3 different alpha values (panel (b), (c) and (d)) and results from our proposed approach (panel (a)). For our approach, the mean of the 100 approximations captures the true functions really well, with slight deviations in the middle region (region with high frequency oscillations). The GP result in panel (b) doesn't learn properly with the data for a relatively large alpha value (result of over-penalization). In panel (c), with a considerable small value of alpha ($10^{-4}$), the mean of the approximation is able to capture the variations on the edges considerably well. However, the middle portion still remains underfit. With an even smaller alpha value as shown in panel (d), now the middle portion is appropriately captured. However, on the edges we have very large fluctuations with predictions completely missing the true function. Hence overall, with GPs (which is also a kernel based approximation), it is really difficult to decide the value of the penalty parameter alpha and a lot of experimentation is required to obtain an acceptable solution. However as seen from panel (a), by using the same setup as all other previous examples, we are able to get a very stable (seen from relatively thinner standard deviation bounds) and accurate (comparing blue and red curves) approximations with our proposed approach. 

From our experiments the multiscale approach is very stable for datasets which are noise free, irrespective of the sampling design. Hence for such cases, if we don't care about the size of the sparse representation and just want to generate an efficient model, we can directly pick a very large scale and not worry about the model overfitting (creating spurious fluctuations and artifacts). For example, for the analysis in Fig. \ref{Pic7}, we have used $\omega = 15$ as our termination scale directly without carrying out cross validation for each of the 100 datasets. The reason for this behavior can be attributed to the extremely local (numerically) support of the functions at higher scales.

\subsection{Application on other datasets}

\begin{figure}[h]
\centering
\includegraphics[width=9.5cm]{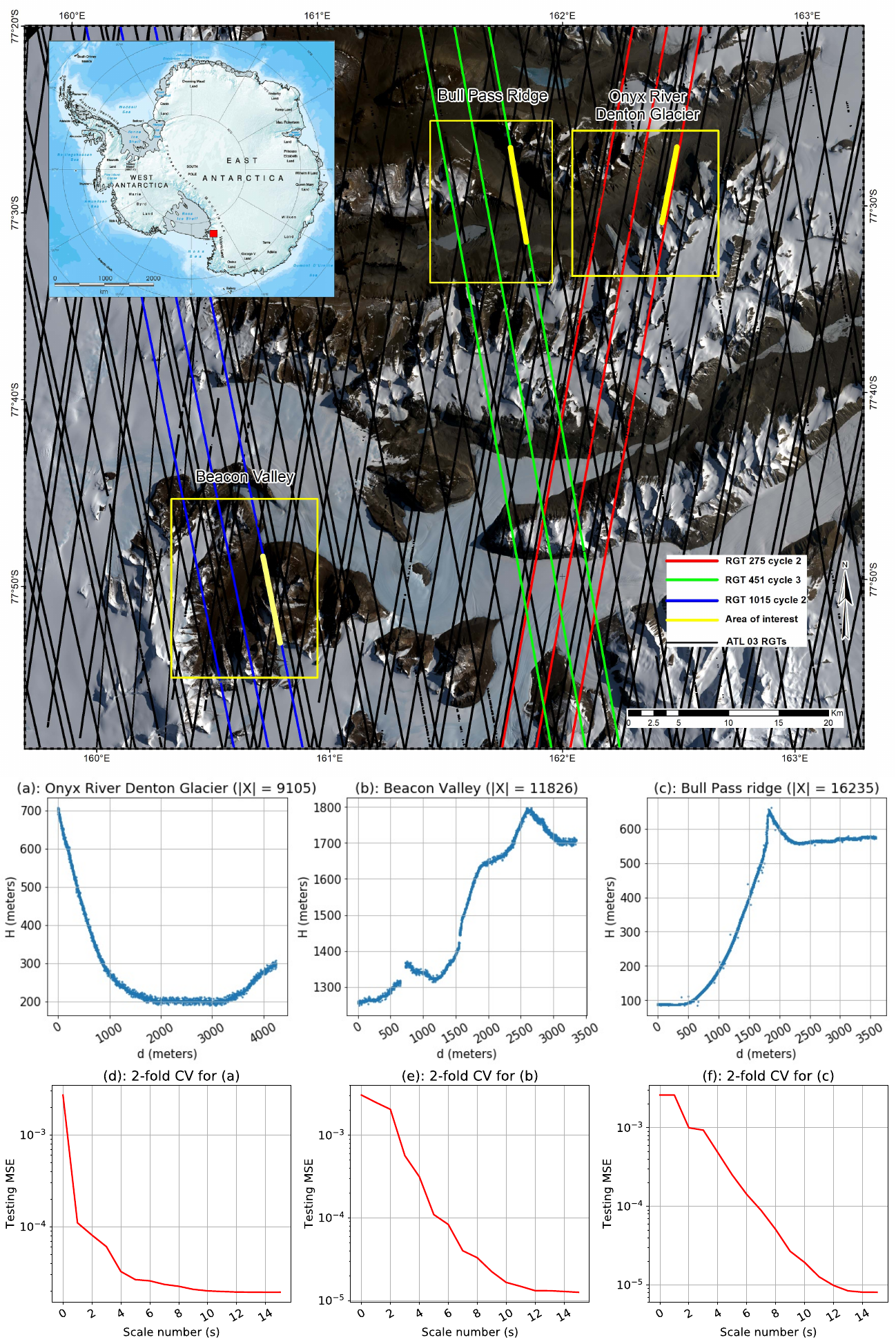}
\caption{\textit{Map}: Location for the section of the 3 ICESat-2 tracks we analyze in this paper. The data was obtained from the National Snow and Ice Data Center (https://nsidc.org/ \cite{nsidcdata}). For map preparation, Landsat-8 2020/02/03, composite bands 4,3,2 image was used from Earth Explorer (https://earthexplorer.usgs.gov/). The bold yellow lines on the 3 tracks in the chosen sites (highlighted by yellow rectangles) show the section of the track where the data is extracted from. \textit{Plots (a) to (c)}: show the distribution of the photons (height vs distance) for these tracks. \textit{Plots (d) to (f)}: show the corresponding 2-CV testing MSE for choosing the truncation scale.}
\label{Pic9}
\end{figure}

In this section we analyze the behavior of our proposed approach on 2 other datasets. We begin by implementing our approach on gray scale face images of 5 different people. The analysis on this dataset is important as image data is fundamentally different than other type of analytical functions or topographic data we analyzed in the previous section. For image data, based on the boundary of features (for example nose, eyes, background etc), capturing the efficient sparse representation for generating good reconstructions (including sharp changes) is a non-trivial task. Fig \ref{Pic8} shows the results for successful application of the proposed approach on this dataset. Here the top row shows the original images ($64 \times 64 = 4096$ pixels). Following this, the subsequent rows show the corresponding reconstruction with different truncation scales (4, 8, 10 and 12 for rows 2,3,4 and 5 respectively). The header for each of the image also shows the corresponding size of the sparse representation that was generated and used in producing these face image approximations. At scale 12, with considerable less number of data points (relative to 4096), the proposed approach was able to produce acceptable reconstructions. It is important to note here that all the reconstructions are done on a $100 \times 100$ resolution (10000 data points) and hence predicting at previously unseen locations. The ability of our approach to capture the information in the data with efficient reconstruction of edges at unseen locations clearly demonstrate the good modeling capabilities of the mutiscale approach.

\begin{figure}[h]
\centering
\includegraphics[width=10.5cm]{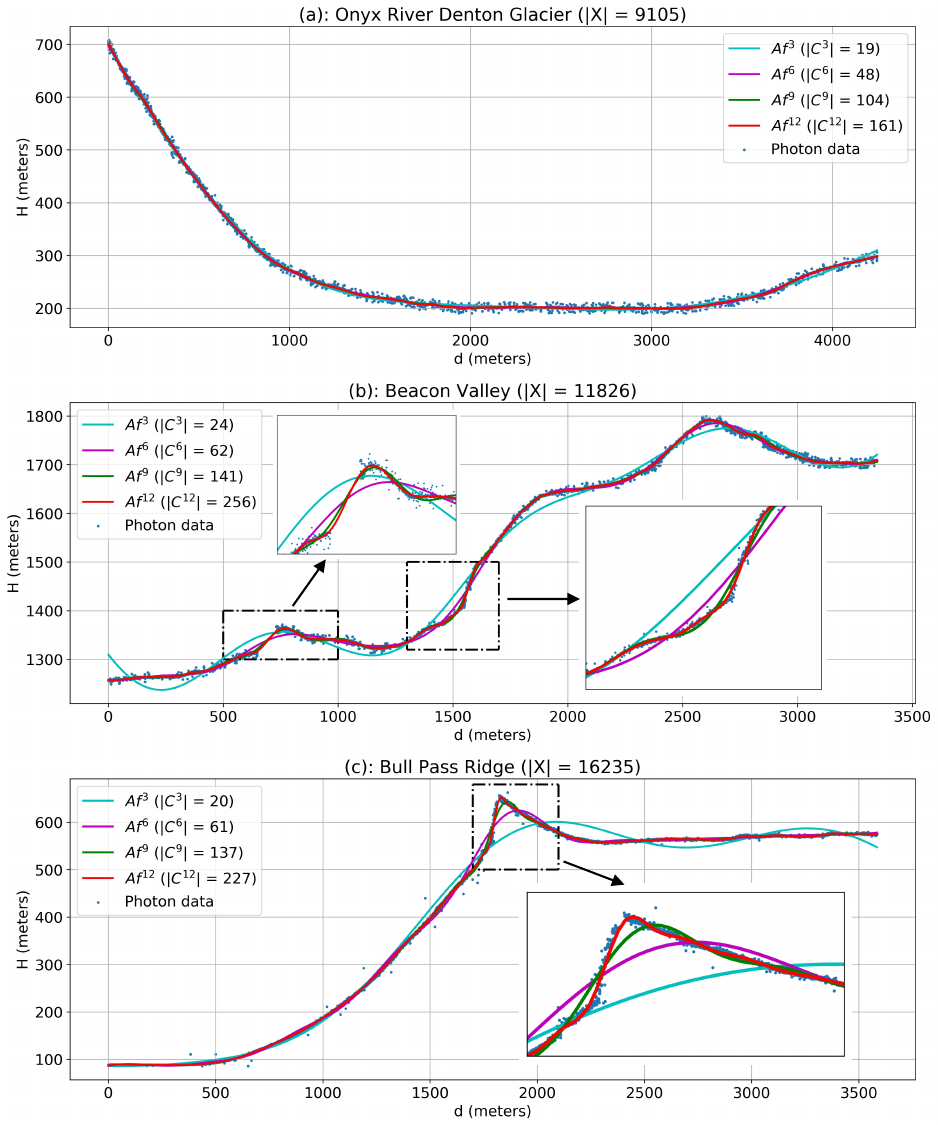}
\caption{Performance of the multiscale approach with reconstruction at scale 3,6,9 and 12 for sites (a), (b) and (c) shown in Fig. \ref{Pic9}. The regions with complication behavior of data are magnified in the insets to clearly demonstrate the quality of the approximation. The legends in each plot also show the size of the sparse representation for each reconstruction to show the quality of data reduction associated with each of the produced approximations.}
\label{Pic10}
\end{figure}

For the final case study, we analyze the performance of our approach on a remote sensing dataset of photon clouds. Specifically, we model the data collected from the still ongoing ICESat-2 mission \cite{markus2017ice}, which has 3 pairs of lasers that map the topography of the earth surface. Depending on the location in the 4-d space (including time) from where the photons are reflected (computed using traveling time of the photon), the challenge is to infer and map the underlying topography. This whole procedure finds applications in analyzing polar ice-sheets, oceans and forest cover \cite{smith2019land,klotz2020high,neuenschwander2019atl08}. Fig. \ref{Pic9} shows three sites which we use in this study from the well known McMurdo Dry Valley region in Antarctica. This site is commonly used for calibration and validation of the satellite instruments due to its stable terrain. The photon cloud for the chosen three sites \cite{10.1002/essoar.10505019.1} are shown in panel (a), (b) and (c) in Fig. \ref{Pic9}. The location of the specific tracks are shown in the map in Fig. \ref{Pic9}. Looking at the sites, (a) with 9105 photons is a relatively smooth region. (b) and (c) on the other hand have higher number of photons with complicated features. Particularly (b) also contains an obvious discontinuity at around 700m. Like the previous experiments, here also we use 2-fold CV and obtain the decaying behavor of the testing MSE shown in Fig. \ref{Pic9} (panels (d), (e) and (f)). Following this, Fig. \ref{Pic10} shows the reconstructions at different scales (3,6,9 and 12) for all the 3 tracks. The figure magnifies complicated regions of the track for the ease of analysis. The most important thing to note here is the highly efficient data reduction due to redundancy in the number of data points available to infer the surface. Particularly, if we assume scale 12 to be our truncation scale based on CV curves in Fig. \ref{Pic9}, with acceptable reconstructions we get a data reduction of 98.2\% (1-161/9105), 97.8\%(1-256/11826) and 98.6\%(1-227/16235) for sites (a), (b) and (c) respectively, as shown in Fig. \ref{Pic10}.

\section{Conclusion and future work}

In this paper we introduced a multiscale approximation space constructed with functions from multiple Reproducing Kernel Hilbert Spaces. For regression in this space, we proposed a greedy approach, that besides giving an acceptable approximation for data also leads to data reduction through computation of efficient sparse representations. In the analysis section we provided results for intelligently deciding the hyperparameters involved in the algorithm. We also further analyzed the approximation properties of finite scale truncation for the overall multiscale kernel. In the results section we analyzed the performance of the algorithm on a wide variety of datasets. For each of these experimental setting, the approach was shown to successfully model the provided dataset while learning efficient sparse representations leading to considerable data reduction (for example in the last remote sensing application, our approach was demonstrated to have a data reduction of almost 2 orders of magnitude).

Regarding future work, one very important direction is to analyze datasets with large noise levels and understand the role of regularization in the multiscale setting. Extension to very large datasets through distributed memory implementations will also make this approach more useful for Big data problems. Finally, in one of the following future works we also plan to analyze and quantify the uncertainty in the final approximation based on the chosen sparse representation at the different considered scales.

\section*{Declarations}

\textbf{Acknowledgments}: We thank Beata Csatho from Department of Geology at University at Buffalo for site selection, data extraction and evaluation of results for the ICESat-2 case study shown in Fig. \ref{Pic9} and \ref{Pic10}. We also thank Ivan Parmuzin from University at Buffalo for data preprocessing and map compilation (Fig. \ref{Pic9}).

\noindent \textbf{Funding}: The work was supported by grant number: OAC2004302

\noindent \textbf{Conflicts of interest/Competing interests}: No conflict of interest.

\noindent \textbf{Availability of data and material}: Not applicable.

\noindent \textbf{Code availability}: \url{https://github.com/pshekhar-tufts/Muliscale_code}

\noindent \textbf{Ethics approval}: : Not applicable.

\noindent \textbf{Consent to participate}: Not applicable.

\noindent \textbf{Consent for publication}: Not applicable.


\bibliographystyle{spmpsci}
\bibliography{ACM}

\begin{thebibliography}{10}
\providecommand{\url}[1]{{#1}}
\providecommand{\urlprefix}{URL }
\expandafter\ifx\csname urlstyle\endcsname\relax
  \providecommand{\doi}[1]{DOI~\discretionary{}{}{}#1}\else
  \providecommand{\doi}{DOI~\discretionary{}{}{}\begingroup
  \urlstyle{rm}\Url}\fi

\bibitem{allard2012multi}
Allard, W.K., Chen, G., Maggioni, M.: Multi-scale geometric methods for data
  sets ii: Geometric multi-resolution analysis.
\newblock Applied and Computational Harmonic Analysis \textbf{32}(3), 435--462
  (2012)

\bibitem{bermanis2013multiscale}
Bermanis, A., Averbuch, A., Coifman, R.R.: Multiscale data sampling and
  function extension.
\newblock Applied and Computational Harmonic Analysis \textbf{34}(1), 15--29
  (2013)

\bibitem{buehlmann2006boosting}
Buehlmann, P., et~al.: Boosting for high-dimensional linear models.
\newblock The Annals of Statistics \textbf{34}(2), 559--583 (2006)

\bibitem{candes2007dantzig}
Candes, E., Tao, T., et~al.: The dantzig selector: Statistical estimation when
  p is much larger than n.
\newblock The annals of Statistics \textbf{35}(6), 2313--2351 (2007)

\bibitem{chen2013robust}
Chen, J., Yang, J.: Robust subspace segmentation via low-rank representation.
\newblock IEEE transactions on cybernetics \textbf{44}(8), 1432--1445 (2013)

\bibitem{couvreur2000optimality}
Couvreur, C., Bresler, Y.: On the optimality of the backward greedy algorithm
  for the subset selection problem.
\newblock SIAM Journal on Matrix Analysis and Applications \textbf{21}(3),
  797--808 (2000)

\bibitem{de2010stability}
De~Marchi, S., Schaback, R.: Stability of kernel-based interpolation.
\newblock Advances in Computational Mathematics \textbf{32}(2), 155--161 (2010)

\bibitem{donoho2006compressed}
Donoho, D.L.: Compressed sensing.
\newblock IEEE Transactions on information theory \textbf{52}(4), 1289--1306
  (2006)

\bibitem{donoho2005stable}
Donoho, D.L., Elad, M., Temlyakov, V.N.: Stable recovery of sparse overcomplete
  representations in the presence of noise.
\newblock IEEE Transactions on information theory \textbf{52}(1), 6--18 (2005)

\bibitem{efendiev2013generalized}
Efendiev, Y., Galvis, J., Hou, T.Y.: Generalized multiscale finite element
  methods (gmsfem).
\newblock Journal of Computational Physics \textbf{251}, 116--135 (2013)

\bibitem{efendiev2009multiscale}
Efendiev, Y., Hou, T.Y.: Multiscale finite element methods: theory and
  applications, vol.~4.
\newblock Springer Science \& Business Media (2009)

\bibitem{elad2010sparse}
Elad, M.: Sparse and redundant representations: from theory to applications in
  signal and image processing.
\newblock Springer Science \& Business Media (2010)

\bibitem{farahat2011novel}
Farahat, A., Ghodsi, A., Kamel, M.: A novel greedy algorithm for nystr{\"o}m
  approximation.
\newblock In: Proceedings of the Fourteenth International Conference on
  Artificial Intelligence and Statistics, pp. 269--277 (2011)

\bibitem{fasshauer2007meshfree}
Fasshauer, G.E.: Meshfree approximation methods with MATLAB, vol.~6.
\newblock World Scientific (2007)

\bibitem{fasshauer2009preconditioning}
Fasshauer, G.E., Zhang, J.G.: Preconditioning of radial basis function
  interpolation systems via accelerated iterated approximate moving least
  squares approximation.
\newblock In: Progress on Meshless Methods, pp. 57--75. Springer (2009)

\bibitem{ferrari2004multiscale}
Ferrari, S., Maggioni, M., Borghese, N.A.: Multiscale approximation with
  hierarchical radial basis functions networks.
\newblock IEEE Transactions on Neural Networks \textbf{15}(1), 178--188 (2004)

\bibitem{gittens2013revisiting}
Gittens, A., Mahoney, M.: Revisiting the nystrom method for improved
  large-scale machine learning.
\newblock In: International Conference on Machine Learning, pp. 567--575. PMLR
  (2013)

\bibitem{goodfellow2016deep}
Goodfellow, I., Bengio, Y., Courville, A., Bengio, Y.: Deep learning, vol.~1.
\newblock MIT press Cambridge (2016)

\bibitem{gribonval2015sample}
Gribonval, R., Jenatton, R., Bach, F., Kleinsteuber, M., Seibert, M.: Sample
  complexity of dictionary learning and other matrix factorizations.
\newblock IEEE Transactions on Information Theory \textbf{61}(6), 3469--3486
  (2015)

\bibitem{griebel2015multiscale}
Griebel, M., Rieger, C., Zwicknagl, B.: Multiscale approximation and
  reproducing kernel hilbert space methods.
\newblock SIAM Journal on Numerical Analysis \textbf{53}(2), 852--873 (2015)

\bibitem{jones1992simple}
Jones, L.K., et~al.: A simple lemma on greedy approximation in hilbert space
  and convergence rates for projection pursuit regression and neural network
  training.
\newblock The annals of Statistics \textbf{20}(1), 608--613 (1992)

\bibitem{klotz2020high}
Klotz, B.W., Neuenschwander, A., Magruder, L.A.: High-resolution ocean wave and
  wind characteristics determined by the icesat-2 land surface algorithm.
\newblock Geophysical Research Letters \textbf{47}(1), e2019GL085907 (2020)

\bibitem{lecun2015deep}
LeCun, Y., Bengio, Y., Hinton, G.: Deep learning.
\newblock nature \textbf{521}(7553), 436--444 (2015)

\bibitem{lewicki2000learning}
Lewicki, M.S., Sejnowski, T.J.: Learning overcomplete representations.
\newblock Neural computation \textbf{12}(2), 337--365 (2000)

\bibitem{liao2016adaptive}
Liao, W., Maggioni, M.: Adaptive geometric multiscale approximations for
  intrinsically low-dimensional data.
\newblock arXiv preprint arXiv:1611.01179  (2016)

\bibitem{lu2006combined}
Lu, L., Vidal, R.: Combined central and subspace clustering for computer vision
  applications.
\newblock In: Proceedings of the 23rd international conference on Machine
  learning, pp. 593--600 (2006)

\bibitem{ma2007segmentation}
Ma, Y., Derksen, H., Hong, W., Wright, J.: Segmentation of multivariate mixed
  data via lossy data coding and compression.
\newblock IEEE transactions on pattern analysis and machine intelligence
  \textbf{29}(9), 1546--1562 (2007)

\bibitem{ma2008estimation}
Ma, Y., Yang, A.Y., Derksen, H., Fossum, R.: Estimation of subspace
  arrangements with applications in modeling and segmenting mixed data.
\newblock SIAM review \textbf{50}(3), 413--458 (2008)

\bibitem{maggioni2016multiscale}
Maggioni, M., Minsker, S., Strawn, N.: Multiscale dictionary learning:
  non-asymptotic bounds and robustness.
\newblock The Journal of Machine Learning Research \textbf{17}(1), 43--93
  (2016)

\bibitem{mairal2009supervised}
Mairal, J., Ponce, J., Sapiro, G., Zisserman, A., Bach, F.R.: Supervised
  dictionary learning.
\newblock In: Advances in neural information processing systems, pp. 1033--1040
  (2009)

\bibitem{mallat1993matching}
Mallat, S.G., Zhang, Z.: Matching pursuits with time-frequency dictionaries.
\newblock IEEE Transactions on signal processing \textbf{41}(12), 3397--3415
  (1993)

\bibitem{markus2017ice}
Markus, T., Neumann, T., Martino, A., Abdalati, W., Brunt, K., Csatho, B.,
  Farrell, S., Fricker, H., Gardner, A., Harding, D., et~al.: The ice, cloud,
  and land elevation satellite-2 (icesat-2): science requirements, concept, and
  implementation.
\newblock Remote Sensing of Environment \textbf{190}, 260--273 (2017)

\bibitem{maurer2010k}
Maurer, A., Pontil, M.: $ k $-dimensional coding schemes in hilbert spaces.
\newblock IEEE Transactions on Information Theory \textbf{56}(11), 5839--5846
  (2010)

\bibitem{mysore2012block}
Mysore, G.J.: A block sparsity approach to multiple dictionary learning for
  audio modeling.
\newblock In: Proc. Int. Conf. Machine Learning (ICML) Workshop Sparsity,
  Dictionaries, Projections Machine Learning Signal Processing (2012)

\bibitem{neuenschwander2019atl08}
Neuenschwander, A., Pitts, K.: The atl08 land and vegetation product for the
  icesat-2 mission.
\newblock Remote sensing of environment \textbf{221}, 247--259 (2019)

\bibitem{nsidcdata}
Neumann, T., et~al.: {ATLAS/ICESat-2 L2A Global Geolocated Photon Data, Version
  3.}
\newblock NASA National Snow and Ice Data Center Distributed Active Archive
  Center,Boulder, Colorado USA  (2020).
\newblock \urlprefix\url{https://doi.org/10.5067/ATLAS/ATL03.003}

\bibitem{neumann2019ice}
Neumann, T.A., Martino, A.J., Markus, T., Bae, S., Bock, M.R., Brenner, A.C.,
  Brunt, K.M., Cavanaugh, J., Fernandes, S.T., Hancock, D.W., et~al.: The ice,
  cloud, and land elevation satellite--2 mission: A global geolocated photon
  product derived from the advanced topographic laser altimeter system.
\newblock Remote sensing of environment \textbf{233}, 111325 (2019)

\bibitem{opfer2006multiscale}
Opfer, R.: Multiscale kernels.
\newblock Advances in computational mathematics \textbf{25}(4), 357--380 (2006)

\bibitem{opfer2006tight}
Opfer, R.: Tight frame expansions of multiscale reproducing kernels in sobolev
  spaces.
\newblock Applied and Computational Harmonic Analysis \textbf{20}(3), 357--374
  (2006)

\bibitem{scikit-learn}
Pedregosa, F., Varoquaux, G., Gramfort, A., Michel, V., Thirion, B., Grisel,
  O., Blondel, M., Prettenhofer, P., Weiss, R., Dubourg, V., Vanderplas, J.,
  Passos, A., Cournapeau, D., Brucher, M., Perrot, M., Duchesnay, E.:
  Scikit-learn: Machine learning in {P}ython.
\newblock Journal of Machine Learning Research \textbf{12}, 2825--2830 (2011)

\bibitem{schaback2006kernel}
Schaback, R., Wendland, H.: Kernel techniques: from machine learning to
  meshless methods.
\newblock Acta numerica \textbf{15}, 543 (2006)

\bibitem{scholkopf2002learning}
Sch{\"o}lkopf, B., Smola, A.J., Bach, F., et~al.: Learning with kernels:
  support vector machines, regularization, optimization, and beyond.
\newblock MIT press (2002)

\bibitem{10.1002/essoar.10505019.1}
Shekhar, P., Csatho, B., Schenk, T., Patra, A.: Exploiting the redundancy in
  icesat-2 geolocated photon data (atl03), a multiscale data reduction
  approach.
\newblock Earth and Space Science Open Archive p.~15 (2020).
\newblock \doi{10.1002/essoar.10505019.1}.
\newblock \urlprefix\url{https://doi.org/10.1002/essoar.10505019.1}

\bibitem{shekhar2020hierarchical}
Shekhar, P., Patra, A.: Hierarchical approximations for data reduction and
  learning at multiple scales.
\newblock Foundations of Data Science \textbf{2}(2), 123 (2020)

\bibitem{smith2019land}
Smith, B., Fricker, H.A., Holschuh, N., Gardner, A.S., Adusumilli, S., Brunt,
  K.M., Csatho, B., Harbeck, K., Huth, A., Neumann, T., et~al.: Land ice
  height-retrieval algorithm for nasa's icesat-2 photon-counting laser
  altimeter.
\newblock Remote Sensing of Environment \textbf{233}, 111352 (2019)

\bibitem{simulationlib}
Surjanovic, S., Bingham, D.: Virtual library of simulation experiments: Test
  functions and datasets.
\newblock Retrieved January 12, 2021, from \url{http://www.sfu.ca/~ssurjano}

\bibitem{tropp2004greed}
Tropp, J.A.: Greed is good: Algorithmic results for sparse approximation.
\newblock IEEE Transactions on Information theory \textbf{50}(10), 2231--2242
  (2004)

\bibitem{vidal2005generalized}
Vidal, R., Ma, Y., Sastry, S.: Generalized principal component analysis (gpca).
\newblock IEEE transactions on pattern analysis and machine intelligence
  \textbf{27}(12), 1945--1959 (2005)

\bibitem{wendland2004scattered}
Wendland, H.: Scattered data approximation, vol.~17.
\newblock Cambridge university press (2004)

\bibitem{williams2000using}
Williams, C., Seeger, M.: Using the nystr{\"o}m method to speed up kernel
  machines.
\newblock Advances in neural information processing systems \textbf{13},
  682--688 (2000)

\bibitem{williams2006gaussian}
Williams, C.K., Rasmussen, C.E.: Gaussian processes for machine learning,
  vol.~2.
\newblock MIT press Cambridge, MA (2006)

\bibitem{zhang2008improved}
Zhang, K., Tsang, I.W., Kwok, J.T.: Improved nystr{\"o}m low-rank approximation
  and error analysis.
\newblock In: Proceedings of the 25th international conference on Machine
  learning, pp. 1232--1239 (2008)

\bibitem{zhang2009consistency}
Zhang, T.: On the consistency of feature selection using greedy least squares
  regression.
\newblock Journal of Machine Learning Research \textbf{10}(3) (2009)

\bibitem{zhang2011adaptive}
Zhang, T.: Adaptive forward-backward greedy algorithm for learning sparse
  representations.
\newblock IEEE transactions on information theory \textbf{57}(7), 4689--4708
  (2011)

\end{thebibliography}

\end{document}